\documentclass{svjour3}

\smartqed

\usepackage{natbib}

\usepackage{graphicx}
\usepackage{url}
\usepackage[tight,footnotesize]{subfigure}
\usepackage[ruled,lined,linesnumbered]{algorithm2e}
\usepackage{cite}
\usepackage{amsmath,amssymb}
\usepackage{cases}

\DeclareMathOperator*{\argmin}{\arg\!\min}

\begin{document}

\title{On Constrained Spectral Clustering and Its Applications}

\author{Xiang Wang \and Buyue Qian \and Ian Davidson}

\institute{
	X. Wang \at Department of Computer Science, University of California, Davis. Davis, CA 95616, USA.\\
    \email{xiang@ucdavis.edu}
	\and
	B. Qian \at Department of Computer Science, University of California, Davis. Davis, CA 95616, USA.\\
    \email{byqian@ucdavis.edu}
    \and
    I. Davidson \at Department of Computer Science, University of California, Davis. Davis, CA 95616, USA.\\
    \email{davidson@cs.ucdavis.edu}
}

\date{Received: date / Accepted: date}

\maketitle

\begin{abstract}

Constrained clustering has been well-studied for algorithms such as $K$-means and hierarchical clustering. However, how to satisfy many constraints in these algorithmic settings has been shown to be intractable. One alternative to encode many constraints is to use spectral clustering, which remains a developing area. In this paper, we propose a flexible framework for constrained spectral clustering. In contrast to some previous efforts that implicitly encode Must-Link and Cannot-Link constraints by modifying the graph Laplacian or constraining the underlying eigenspace, we present a more natural and principled formulation, which explicitly encodes the constraints as part of a constrained optimization problem. Our method offers several practical advantages: it can encode the degree of belief in Must-Link and Cannot-Link constraints; it guarantees to lower-bound how well the given constraints are satisfied using a user-specified threshold; it can be solved deterministically in polynomial time through generalized eigendecomposition. Furthermore, by inheriting the objective function from spectral clustering and encoding the constraints explicitly, much of the existing analysis of unconstrained spectral clustering techniques remains valid for our formulation. We validate the effectiveness of our approach by empirical results on both artificial and real datasets. We also demonstrate an innovative use of encoding large number of constraints: transfer learning via constraints.

\end{abstract}

\spnewtheorem{prop}{Property}{\bf}{\it}
\spnewtheorem{lem}{Lemma}{\bf}{\it}

\section{Introduction}
\label{sec:intro}

\subsection{Background and Motivation}

Spectral clustering is an important clustering technique that has been extensively studied in the image processing, data mining, and machine learning communities (\citet{Shi2000,Luxburg2007,Ng2001}). It is considered superior to traditional clustering algorithms like $K$-means in terms of having deterministic polynomial-time solution, the ability to model arbitrary shaped clusters, and its equivalence to certain graph cut problems. For example, spectral clustering is able to capture the underlying moon-shaped clusters as shown in Fig.~\ref{fig:two_moon}(b), whereas $K$-means would fail (Fig.~\ref{fig:two_moon}(a)). The advantage of spectral clustering has also been validated by many real-world applications, such as image segmentation (\citet{Shi2000}) and mining social networks (\citet{White2005}).

\begin{figure*}[t]
\centering
\subfigure[K-means]{\includegraphics*[width=0.45\linewidth]{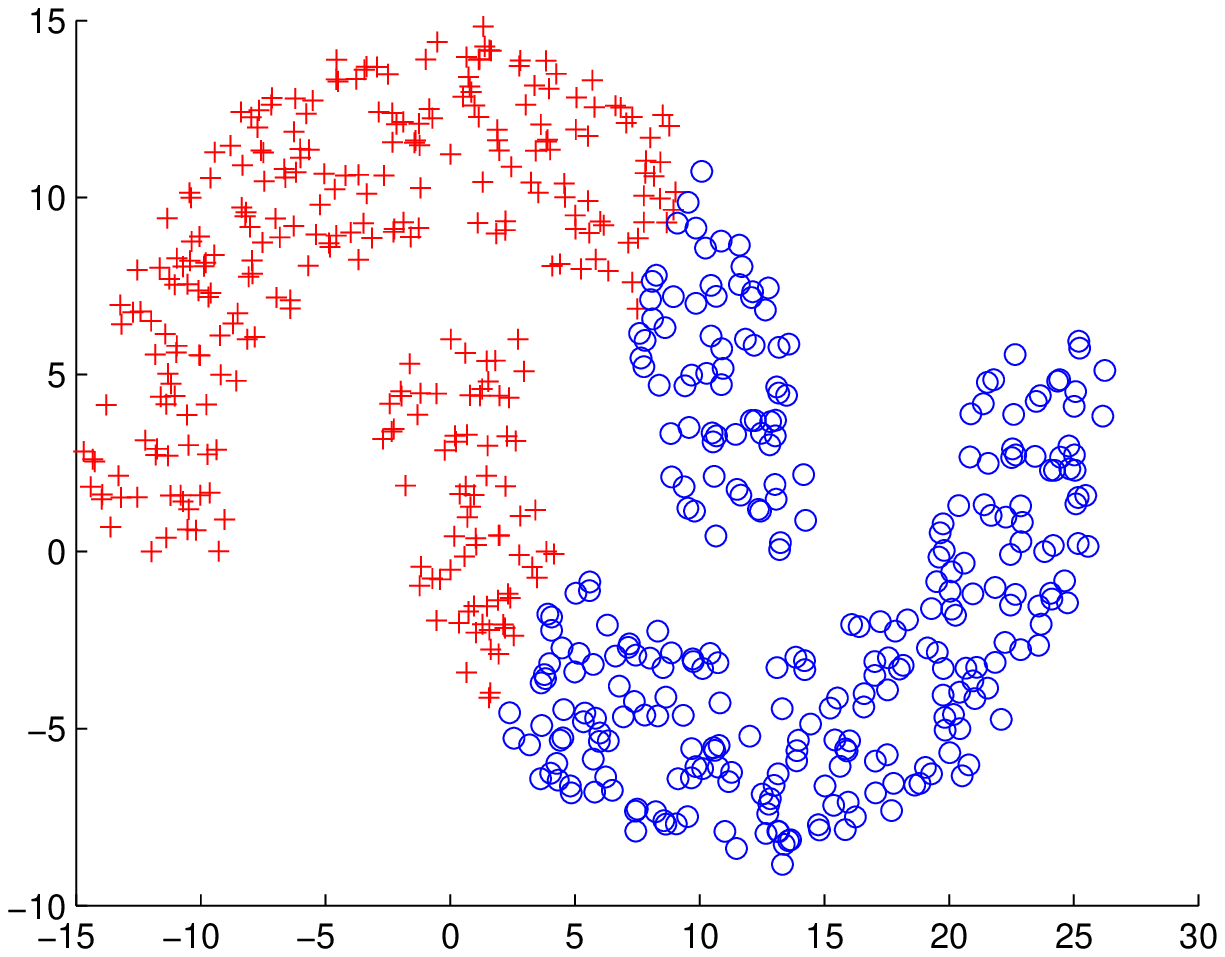}}
\subfigure[Spectral clustering]{\includegraphics*[width=0.45\linewidth]{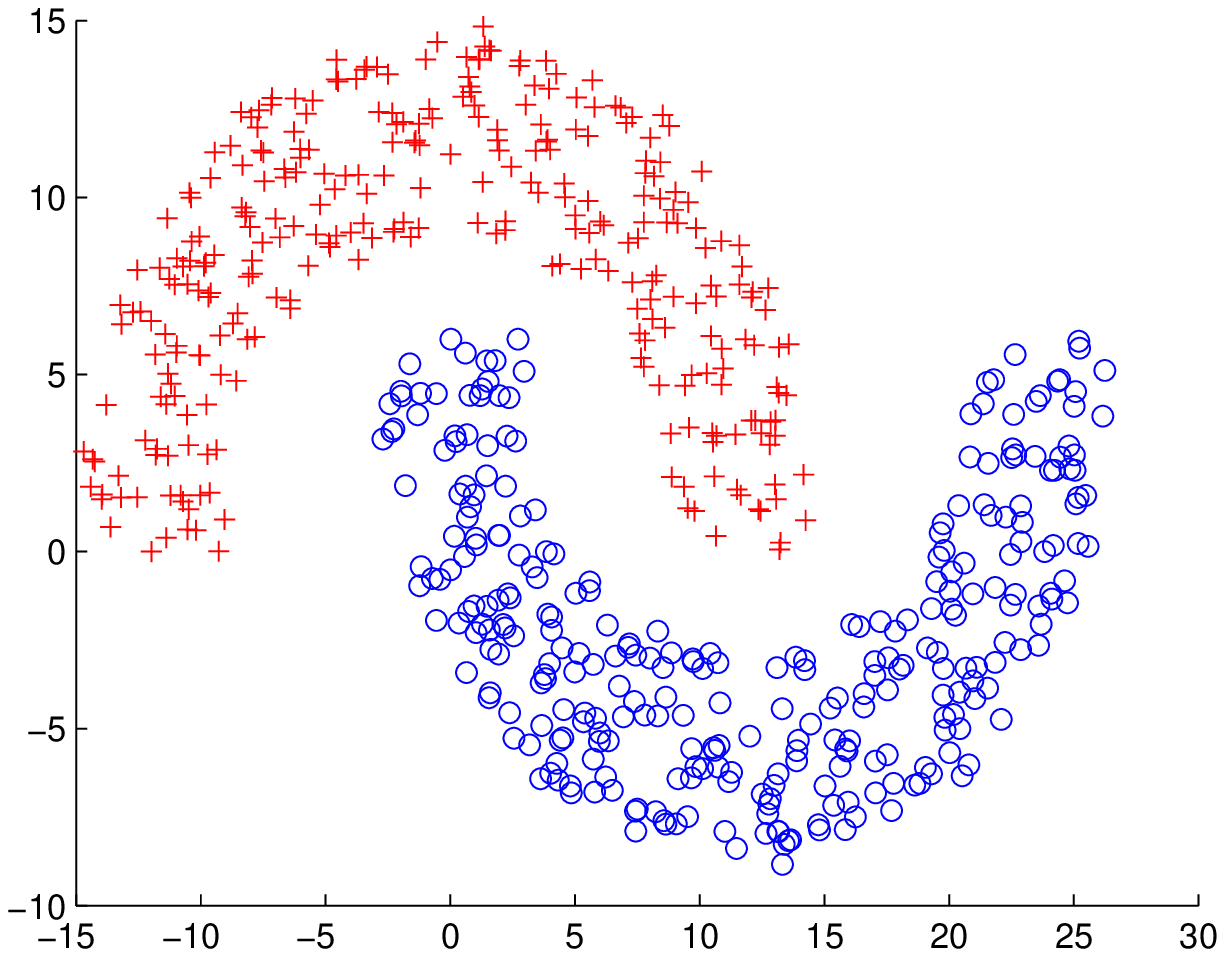}}\\
\subfigure[Spectral clustering]{\includegraphics*[width=0.45\linewidth]{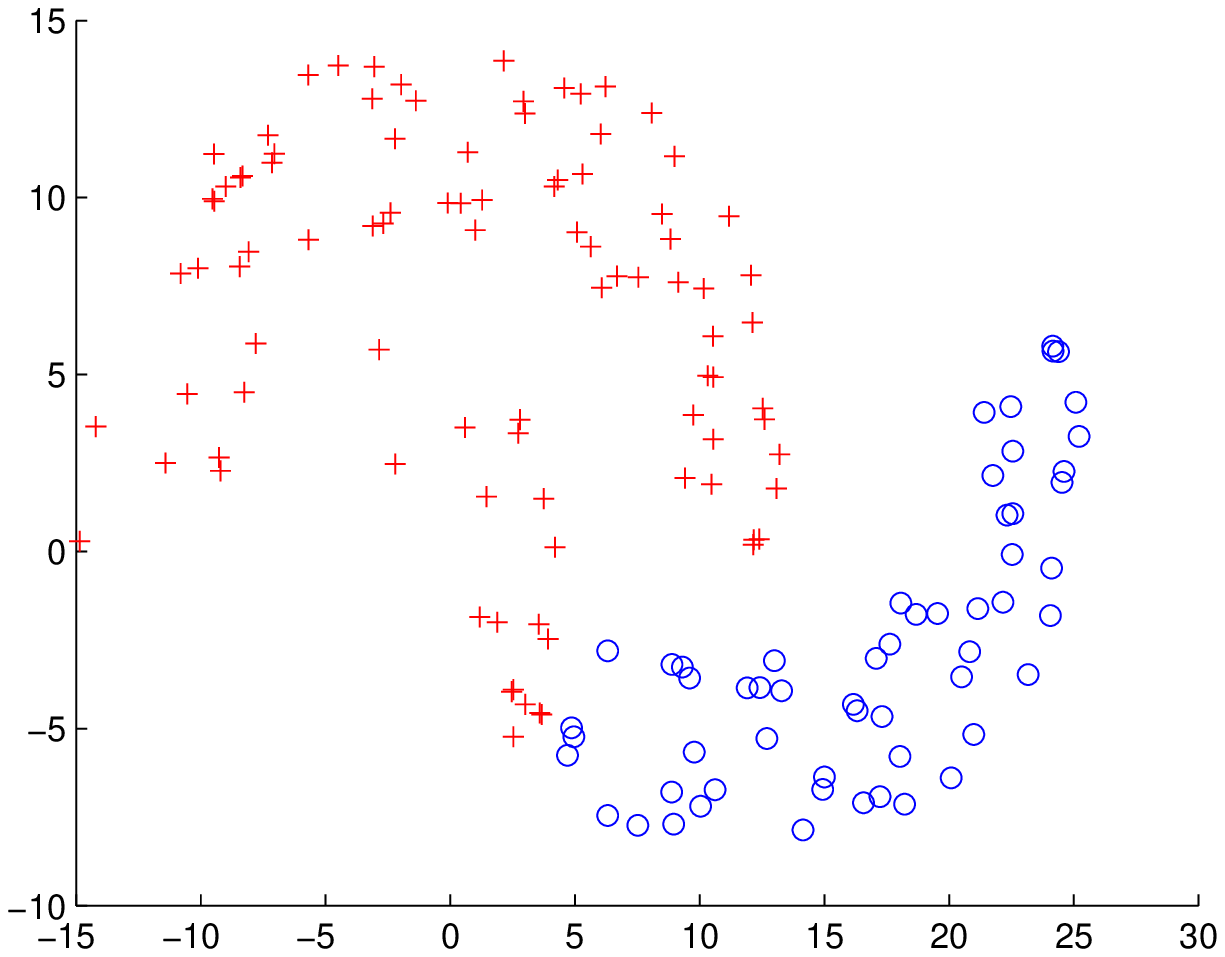}}
\subfigure[Constrained spectral clustering]{\includegraphics*[width=0.45\linewidth]{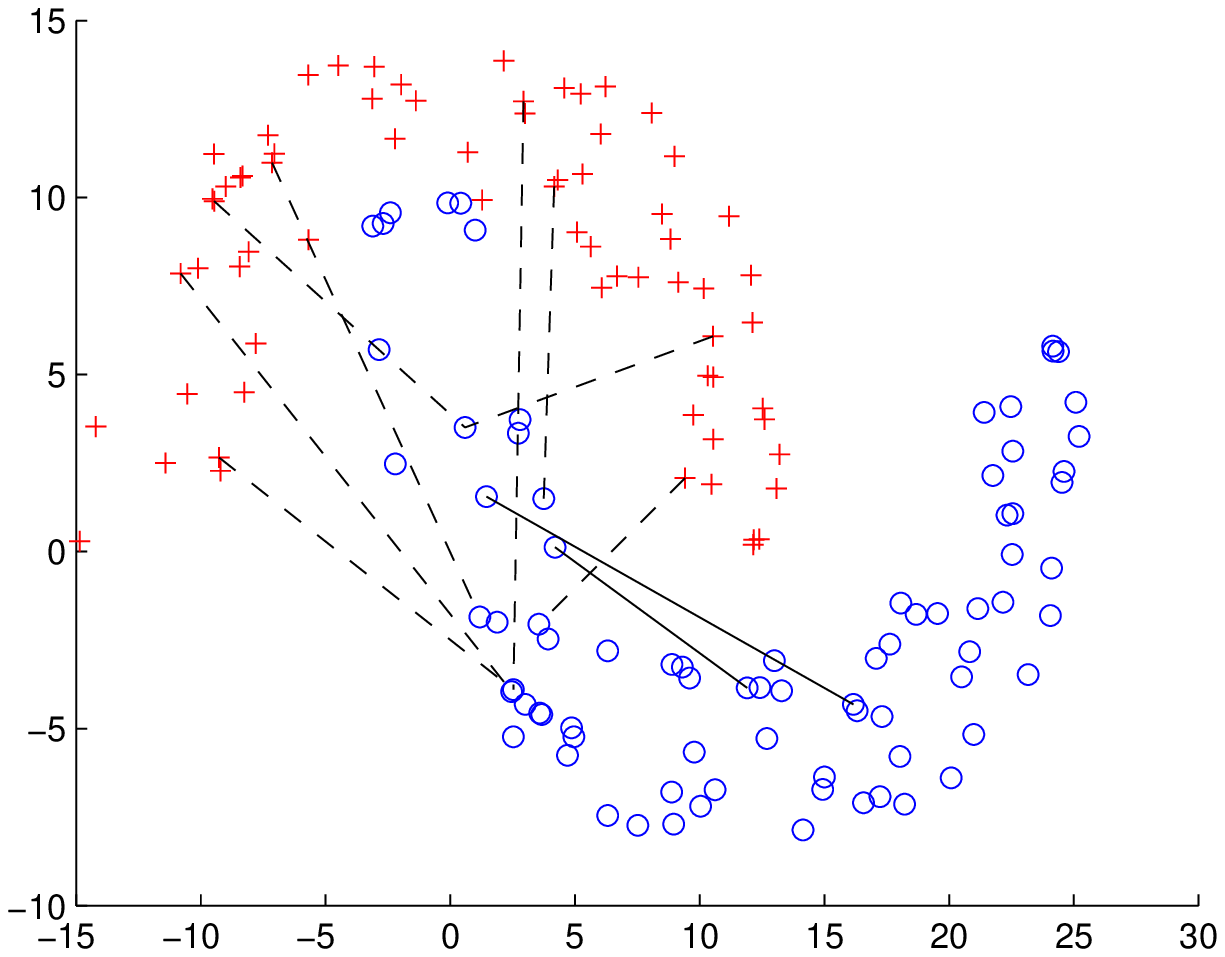}}
\caption{A motivating example for constrained spectral clustering.}\label{fig:two_moon}
\end{figure*}

Spectral clustering was originally proposed to address an unsupervised learning problem: the data instances are unlabeled, and all available information is encoded in the graph Laplacian. However, there are cases where unsupervised spectral clustering becomes insufficient. Using the same toy data, as shown in (Fig.~\ref{fig:two_moon}(c)), when the two moons are under-sampled, the clusters become so sparse that the separation of them becomes difficult. To help spectral clustering recover from an undesirable partition, we can introduce side information in various forms, in either small or large amounts. For example:
\begin{enumerate}
\item \textbf{Pairwise constraints}: Domain experts may explicitly assign constraints that state a pair of instances must be in the same cluster (Must-Link, ML for short) or that a pair of instances cannot be in the same cluster (Cannot-Link, CL for short). For instance, as shown in Fig.~\ref{fig:two_moon}(d), we assigned several ML (solid lines) and CL (dashed lines) constraints, then applied our constrained spectral clustering algorithm, which we will describe later. As a result, the two moons were successfully recovered.
\item \textbf{Partial labeling}: There can be labels on some of the instances, which are neither complete nor exhaustive. We demonstrate in Fig.~\ref{fig:uci_ari} that even small amounts of labeled information can greatly improve clustering results when compared against the ground truth partition, as inferred by the labels.
\item \textbf{Alternative weak distance metrics}: In some situations there may be more than one distance metrics available. For example, in Table~\ref{table:reuters} and accompanying paragraphs we describe clustering documents using distance functions based on different languages (features).
\item \textbf{Transfer of knowledge}: In the context of transfer learning (\citet{DBLP:journals/tkde/PanY10}), if we treat the graph Laplacian as the target domain, we could transfer knowledge from a different but related graph, which can be viewed as the source domain. We discuss this direction in Section~\ref{sec:algorithm:transfer} and \ref{sec:exp:transfer}.
\end{enumerate}
All the aforementioned side information can be transformed into pairwise ML and CL constraints, which could either be hard (binary) or soft (degree of belief). For example, if the side information comes from a source graph, we can construct pairwise constraints by assuming that the more similar two instance are in the source graph, the more likely they belong to the same cluster in the target graph. Consequently the constraints should naturally be represented by a degree of belief, rather than a binary assertion.

How to make use of these side information to improve clustering falls into the area of constrained clustering (\citet{Basu2008}). In general, constrained clustering is a category of techniques that try to incorporate ML and CL constraints into existing clustering schemes. It has been well studied on algorithms such as $K$-means clustering, mixture model, hierarchical clustering, and density-based clustering. Previous studies showed that satisfying all constraints at once (\citet{DBLP:journals/datamine/DavidsonR07}), incrementally (\citet{DBLP:conf/kdd/DavidsonRE07}), or even pruning constraints (\citet{DBLP:conf/icml/DavidsonR07}) is intractable. Furthermore, it was shown that algorithms that build set partitions incrementally (such as $K$-means and EM) are prone to being over-constrained (\citet{DBLP:conf/aaai/DavidsonR06}). In contrast, incorporating constraints into spectral clustering is a promising direction since, unlike existing algorithms, all data instances are assigned simultaneously to clusters, even if the given constraints are inconsistent.

Constrained spectral clustering is still a developing area. Previous work on this topic can be divided into two categories, based on how they enforce the constraints. The first category (\citet{Kamvar2003,Xu2005,Lu2008,Wang2009,Ji2006}) directly manipulate the graph Laplacian (or equivalently, the affinity matrix) according to the given constraints; then unconstrained spectral clustering is applied on the modified graph Laplacian. The second category use constraints to restrict the feasible solution space (\citet{Bie2004,Coleman2008,Li2009,Yu2001,Yu2004}). Existing methods in both categories share several limitations:
\begin{itemize}
\item They are designed to handle only binary constraints. However, as we have stated above, in many real-world applications, constraints are made available in the form of real-valued degree of belief, rather than a yes or no assertion.
\item They aim to satisfy as many constraints as possible, which could lead to inflexibility in practice. For example, the given set of constraints could be noisy, and satisfying some of the constraints could actually hurt the overall performance. Also, it is reasonable to ignore a small portion of constraints in exchange for a clustering with much lower cost.
\item They do not offer any natural interpretation of either the way that constraints are encoded or the implication of enforcing them.
\end{itemize}

\subsection{Our Contributions}

In this paper, we study how to incorporate \textbf{large} amounts of pairwise constraints into spectral clustering, in a flexible manner that addresses the limitations of previous work. Then we show the practical benefits of our approach, including new applications previously not possible.

We extend beyond binary ML/CL constraints and propose a more flexible framework to accommodate general-type side information. We allow the binary constraints to be relaxed to real-valued degree of belief that two data instances belong to the same cluster or two different clusters. Moreover, instead of trying to satisfy each and every constraint that has been given, we use a user-specified threshold to lower bound how well the given constraints must be satisfied. Therefore, \textbf{our method provides maximum flexibility in terms of both representing constraints and satisfying them}. This, in addition to handling large amounts of constraints, allows the encoding of new styles of information such as entire graphs and alternative distance metrics in their raw form without considering issues such as constraint inconsistencies and over-constraining.

Our contributions are:
\begin{itemize}
\item We propose a principled framework for constrained spectral clustering that can incorporate large amounts of both hard and soft constraints.
\item We show how to enforce constraints in a flexible way: a user-specified threshold is introduced so that a limited amount of constraints can be ignored in exchange for lower clustering cost. This allows incorporating side information in its raw form without considering issues such as inconsistency and over-constraining.
\item We extend the objective function of unconstrained spectral clustering by encoding constraints explicitly and creating a novel constrained optimization problem. Thus our formulation naturally covers unconstrained spectral clustering as a special case.
\item We show that our objective function can be turned into a generalized eigenvalue problem, which can be solved deterministically in polynomial time. This is a major advantage over constrained $K$-means clustering, which produces non-deterministic solutions while being intractable even for $K=2$ (\citet{Drineas2004,DBLP:conf/icml/DavidsonR07}).
\item We interpret our formulation from both the graph cut perspective and the Laplacian embedding perspective.
\item We validate the effectiveness of our approach and its advantage over existing methods using standard benchmarks and new innovative applications such as transfer learning.
\end{itemize}

This paper is an extension of our previous work (\citet{WangD:KDD10}) with the following additions: 1) we extend our algorithm from 2-way partition to $K$-way partition (Section~\ref{sec:algorithm:k-way}); 2) we add a new geometric interpretation to our algorithm (Section~\ref{sec:interpret:geo}); 3) we show how to apply our algorithm to a novel application (Section~\ref{sec:algorithm:transfer}), namely transfer learning, and test it with a real-world fMRI dataset (Section~\ref{sec:exp:transfer}); 4) we present a much more comprehensive experiment section with more tasks conducted on more datasets (Section~\ref{sec:exp:2moon} and \ref{sec:exp:reuters}).

The rest of the paper is organized as follows: In Section~\ref{sec:related} we briefly survey previous work on constrained spectral clustering; Section~\ref{sec:background} provides preliminaries for spectral clustering; in Section~\ref{sec:model} we formally introduce our formulation for constrained spectral clustering and show how to solve it efficiently; in Section~\ref{sec:interpret} we interpret our objective from two different perspectives; in Section~\ref{sec:algorithm} we discuss the implementation of our algorithm and possible extensions; we empirically evaluate our approach in Section~\ref{sec:exp}; Section~\ref{sec:conclusion} concludes the paper.

\section{Related Work}
\label{sec:related}

Constrained clustering is a category of methods that extend clustering from unsupervised setting to semi-supervised setting, where side information is available in the form of, or can be converted into, pairwise constraints. A number of algorithms have been proposed on how to incorporate constraints into spectral clustering, which can be grouped into two categories.

The first category manipulates the graph Laplacian directly. \citet{Kamvar2003} proposed the spectral learning algorithm that sets the $(i,j)$-th entry of the affinity matrix to 1 if there is a ML between node $i$ and $j$; 0 for CL. A new graph Laplacian is then computed based on the modified affinity matrix. In (\citet{Xu2005}), the constraints are encoded in the same way, but a random walk matrix is used instead of the normalized Laplacian. \citet{DBLP:conf/icml/KulisBDM05} proposed to add both positive (for ML) and negative (for CL) penalties to the affinity matrix (they then used kernel $K$-means, instead of spectral clustering, to find the partition based on the new kernel). \citet{Lu2008} proposed to propagate the constraints in the affinity matrix. In \citet{Ji2006,Wang2009}, the graph Laplacian is modified by combining the constraint matrix as a regularizer. The limitation of these approaches is that there is no principled way to decide the weights of the constraints, and there is no guarantee that how well the give constraints will be satisfied.

The second category manipulates the eigenspace directly. For example, the subspace trick introduced by \citet{Bie2004} alters the eigenspace which the cluster indicator vector is projected onto, based on the given constraints. This technique was later extended in \citet{Coleman2008} to accommodate inconsistent constraints. \citet{Yu2001,Yu2004} encoded partial grouping information as a subspace projection. \citet{Li2009} enforced constraints by regularizing the spectral embedding. This type of approaches usually strictly enforce given constraints. As a result, the results are often over-constrained, which makes the algorithms sensitive to noise and inconsistencies in the constraint set. Moreover, it is non-trivial to extend these approaches to incorporate soft constraints.

In addition, \citet{DBLP:conf/aaai/GuLH11} proposed a spectral kernel design that combines multiple clustering tasks. The learned kernel is constrained in such a way that the data distributions of any two tasks are as close as possible. Their problem setting differs from ours because we aim to perform single-task clustering by using two (disagreeing) data sources. \citet{DBLP:conf/sdm/WangLZ08} showed how to incorporate pairwise constraints into a penalized matrix factorization framework. Their matrix approximation objective function, which is different from our normalized min-cut objective, is solved by an EM-like algorithm.

We would like to stress that the pros and cons of spectral clustering as compared to other clustering schemes, such as $K$-means clustering, hierarchical clustering, etc., have been thoroughly studied and well established. We do not claim that constrained spectral clustering is universally superior to other constrained clustering schemes. The goal of this work is to provide a way to incorporate constraints into spectral clustering that is more flexible and principled as compared with existing constrained spectral clustering techniques.

\section{Background and Preliminaries}
\label{sec:background}

In this paper we follow the standard graph model that is commonly used in the spectral clustering literature. We reiterate some of the definitions and properties in this section, such as graph Laplacian, normalized min-cut, eigendecomposition and so forth, to make this paper self-contained. Readers who are familiar with the materials can skip to our formulation in Section~\ref{sec:model}. Important notations used throughout the rest of the paper are listed in Table~\ref{table:notation}.

\begin{table}
\centering
\caption{Table of notations}\label{table:notation}
\begin{tabular}{|c|l|}
  \hline Symbol & Meaning\\
  \hline
  $\mathcal{G}$ & An undirected (weighted) graph\\
  $A$ & The affinity matrix\\
  $D$ & The degree matrix\\
  $I$ & The identity matrix\\
  $L / \bar{L}$ & The unnormalized/normalized graph Laplacian\\
  $Q / \bar{Q}$ & The unnormalized/normalized constraint matrix\\
  $vol$ & The volume of graph $\mathcal{G}$\\
  \hline
\end{tabular}
\end{table}

A collection of $N$ data instances is modeled by an undirected, weighted graph $\mathcal{G}(\mathcal{V},\mathcal{E},A)$, where each data instance corresponds to a vertex (node) in $\mathcal{V}$; $\mathcal{E}$ is the edge set and $A$ is the associated affinity matrix. $A$ is symmetric and non-negative. The diagonal matrix $D=\text{diag}(D_{11},\ldots,D_{NN})$ is called the degree matrix of graph $\mathcal{G}$, where
\begin{displaymath}
  D_{ii}=\sum_{j=1}^N A_{ij}.
\end{displaymath}
Then
\begin{displaymath}
  L = D - A
\end{displaymath}
is called the unnormalized graph Laplacian of $\mathcal{G}$. Assuming $\mathcal{G}$ is connected (i.e. any node is reachable from any other node), $L$ has the following properties:
\begin{prop}[Properties of graph Laplacian (\citet{Luxburg2007})]
\label{prop:laplacian}
Let $L$ be the graph Laplacian of a connected graph, then:
\begin{enumerate}
\item $L$ is symmetric and positive semi-definite.
\item $L$ has one and only one eigenvalue equal to 0, and $N-1$ positive eigenvalues: $0 = \lambda_0 < \lambda_1 \leq \ldots \leq \lambda_{N-1}$.
\item $\mathbf{1}$ is an eigenvector of $L$ with eigenvalue 0 ($\mathbf{1}$ is a constant vector whose entries are all $1$).
\end{enumerate}
\end{prop}

\citet{Shi2000} showed that the eigenvectors of the graph Laplacian can be related to the normalized min-cut (Ncut) of $\mathcal{G}$. The objective function can be written as:
\begin{equation}\label{eq:spectral_clustering}
\argmin_{\mathbf{v} \in \mathbb{R}^N} \mathbf{v}^T \bar{L} \mathbf{v},~\text{s.t.}~\mathbf{v}^T \mathbf{v}= vol,~\mathbf{v} \perp D^{1/2} \mathbf{1}.
\end{equation}
Here
\begin{displaymath}
\bar{L} = D^{-1/2} L D^{-1/2}
\end{displaymath}
is called the \emph{normalized} graph Laplacian (\citet{Luxburg2007}); $vol=\sum_{i=1}^N D_{ii}$ is the volume of $\mathcal{G}$; the first constraint $\mathbf{v}^T \mathbf{v}= vol$ normalizes $\mathbf{v}$; the second constraint $\mathbf{v} \perp D^{1/2} \mathbf{1}$ rules out the principal eigenvector of $\bar{L}$ as a trivial solution, because it does not define a meaningful cut on the graph. The relaxed cluster indicator $\mathbf{u}$ can be recovered from $\mathbf{v}$ as:
\begin{displaymath}
\mathbf{u} = D^{-1/2} \mathbf{v}.
\end{displaymath}

Note that the result of spectral clustering is solely decided by the affinity structure of graph $\mathcal{G}$ as encoded in the matrix $A$ (and thus the graph Laplacian $L$). We will then describe our extensions on how to incorporate side information so that the result of clustering will reflect both the affinity structure of the graph and the structure of the side information.

\section{A Flexible Framework for Constrained Spectral Clustering}
\label{sec:model}

In this section, we show how to incorporate side information into spectral clustering as pairwise constraints. Our formulation allows both hard and soft constraints. We propose a new constrained optimization formulation for constrained spectral clustering. Then we show how to solve the objective function by converting it into a generalized eigenvalue system.

\subsection{The Objective Function}
\label{sec:model:obj}

We encode side information with an $N \times N$ constraint matrix $Q$. Traditionally, constrained clustering only accommodates binary constraints, namely Must-Link and Cannot-Link:
\begin{displaymath}
  Q_{ij} = Q_{ji} =
  \begin{cases}
    +1 & \text{if}~ML(i,j)\\
    -1 & \text{if}~CL(i,j)\\
    0  & \text{no side information available}\\
  \end{cases}.
\end{displaymath}
Let $\mathbf{u} \in \{-1,+1\}^N$ be a cluster indicator vector, where $\mathbf{u}_i=+1$ if node $i$ belongs to cluster $+$ and $\mathbf{u}_i=-1$ if node $i$ belongs to cluster $-$, then
\begin{displaymath}
\mathbf{u}^T Q \mathbf{u} = \sum_{i=1}^N \sum_{j=1}^N \mathbf{u}_i \mathbf{u}_j Q_{ij}
\end{displaymath}
is a measure of how well the constraints in $Q$ are satisfied by the assignment $\mathbf{u}$: the measure will increase by $1$ if $Q_{ij}=1$ and node $i$ and $j$ have the same sign in $\mathbf{u}$; it will decrease by $1$ if $Q_{ij}=1$ but node $i$ and $j$ have different signs in $\mathbf{u}$ or $Q_{ij}=-1$ but node $i$ and $j$ have the same sign in $\mathbf{u}$.

We extend the above encoding scheme to accommodate soft constraints by relaxing the cluster indicator vector $\mathbf{u}$ as well as the constraint matrix $Q$ such that:
\begin{displaymath}
\mathbf{u} \in \mathbb{R}^N, Q \in \mathbb{R}^{N \times N}.
\end{displaymath}
$Q_{ij}$ is positive if we believe nodes $i$ and $j$ belong to the same cluster; $Q_{ij}$ is negative if we believe nodes $i$ and $j$ belong to different clusters; the magnitude of $Q_{ij}$ indicates how strong the belief is.

Consequently, $\mathbf{u}^T Q \mathbf{u}$ becomes a real-valued measure of how well the constraints in $Q$ are satisfied in the relaxed sense. For example, $Q_{ij}<0$ means we believe nodes $i$ and $j$ belong to different clusters; in order to improve $\mathbf{u}^T Q \mathbf{u}$, we should assign $\mathbf{u}_i$ and $\mathbf{u}_j$ with values of different signs; similarly, $Q_{ij}>0$ means nodes $i$ and $j$ are believed to belong to the same cluster; we should assign $\mathbf{u}_i$ and $\mathbf{u}_j$ with values of the same sign. The larger $\mathbf{u}^T Q \mathbf{u}$ is, the better the cluster assignment $\mathbf{u}$ conforms to the given constraints in $Q$.

Now given this real-valued measure, rather than trying to satisfy all the constraints in $Q$ individually, we can lower-bound this measure with a constant $\alpha \in \mathbb{R}$:
\begin{displaymath}
\mathbf{u}^T Q \mathbf{u} \geq \alpha.
\end{displaymath}
Following the notations in Eq.(\ref{eq:spectral_clustering}), we substitute $\mathbf{u}$ with $D^{-1/2} \mathbf{v}$, above inequality becomes
\begin{displaymath}
\mathbf{v}^T \bar{Q} \mathbf{v} \geq \alpha,
\end{displaymath}
where
\begin{displaymath}
\bar{Q} = D^{-1/2} Q D^{-1/2}
\end{displaymath}
is the \emph{normalized} constraint matrix.

We append this lower-bound constraint to the objective function of unconstrained spectral clustering in Eq.(\ref{eq:spectral_clustering}) and we have:
\begin{problem}[Constrained Spectral Clustering]
\label{problem:constrained}
Given a normalized graph Laplacian $\bar{L}$, a normalized constraint matrix $\bar{Q}$ and a threshold $\alpha$, we want to optimizes the following objective function:
\begin{equation}\label{eq:objective}
\argmin_{\mathbf{v} \in \mathbb{R}^N} \mathbf{v}^T \bar{L} \mathbf{v},~\text{s.t.}~\mathbf{v}^T \bar{Q} \mathbf{v} \geq \alpha,~\mathbf{v}^T \mathbf{v} = vol,~\mathbf{v} \neq D^{1/2} \mathbf{1}.
\end{equation}
\end{problem}
Here $\mathbf{v}^T \bar{L} \mathbf{v}$ is the cost of the cut we want to minimize; the first constraint $\mathbf{v}^T \bar{Q} \mathbf{v} \geq \alpha$ is to lower bound how well the constraints in $Q$ are satisfied; the second constraint $\mathbf{v}^T \mathbf{v} = vol$ normalizes $\mathbf{v}$; the third constraint $\mathbf{v} \neq D^{1/2} \mathbf{1}$ rules out the trivial solution $D^{1/2} \mathbf{1}$. Suppose $\mathbf{v}^*$ is the optimal solution to Eq.(\ref{eq:objective}), then $\mathbf{u}^* = D^{-1/2} \mathbf{v}^*$ is the optimal cluster indicator vector.

It is easy to see that the unconstrained spectral clustering in Eq.(\ref{eq:spectral_clustering}) is covered as a special case of Eq.(\ref{eq:objective}) where $\bar{Q}=I$ (no constraints are encoded) and $\alpha=vol$ ($\mathbf{v}^T \bar{Q} \mathbf{v} \geq vol$ is trivially satisfied given $\bar{Q}=I$ and $\mathbf{v}^T \mathbf{v} = vol$).

\subsection{Solving the Objective Function}
\label{sec:model:sol}

To solve a constrained optimization problem, we follow the Karush-Kuhn-Tucker Theorem (\citet{Kuhn1982}) to derive the necessary conditions for the optimal solution to the problem. We can find a set of candidates, or \emph{feasible solutions}, that satisfy all the necessary conditions. Then we choose the optimal solution among the feasible solutions using brute-force method, given the size of the feasible set is finite and small.

For our objective function in Eq.(\ref{eq:objective}), we introduce the Lagrange multipliers as follows:
\begin{equation}\label{eq:L}
\Lambda(\mathbf{v},\lambda,\mu) = \mathbf{v}^T \bar{L} \mathbf{v} - \lambda (\mathbf{v}^T \bar{Q} \mathbf{v} - \alpha) - \mu (\mathbf{v}^T \mathbf{v} - vol).
\end{equation}
Then according to the KKT Theorem, any feasible solution to Eq.(\ref{eq:objective}) must satisfy the following conditions:
\begin{align}
\mbox{(Stationarity)}~&~\bar{L} \mathbf{v} - \lambda \bar{Q} \mathbf{v} - \mu \mathbf{v} = 0, \label{eq:KKT:stationarity}\\
\mbox{(Primal feasibility)}~&~\mathbf{v}^T \bar{Q} \mathbf{v} \geq \alpha, \mathbf{v}^T \mathbf{v} = vol, \label{eq:KKT:primal}\\
\mbox{(Dual feasibility)}~&~\lambda \geq 0, \label{eq:KKT:dual}\\
\mbox{(Complementary slackness)}~&~\lambda (\mathbf{v}^T \bar{Q} \mathbf{v} - \alpha) = 0. \label{eq:KKT:slackness}
\end{align}
Note that Eq.(\ref{eq:KKT:stationarity}) comes from taking the derivative of Eq.(\ref{eq:L}) with respect to $\mathbf{v}$. Also note that we dismiss the constraint $\mathbf{v} \neq D^{1/2} \mathbf{1}$ at this moment, because it can be checked independently after we find the feasible solutions.

To solve Eq.(\ref{eq:KKT:stationarity})-(\ref{eq:KKT:slackness}), we start with looking at the complementary slackness requirement in Eq.(\ref{eq:KKT:slackness}), which implies either $\lambda=0$ or $\mathbf{v}^T \bar{Q} \mathbf{v} - \alpha = 0$. If $\lambda = 0$, we will have a trivial problem because the second term from Eq.(\ref{eq:KKT:stationarity}) will be eliminated and the problem will be reduced to unconstrained spectral clustering. Therefore in the following we focus on the case where $\lambda \neq 0$. In this case, for Eq.(\ref{eq:KKT:slackness}) to hold $\mathbf{v}^T \bar{Q} \mathbf{v} - \alpha$ must be $0$. Consequently the KKT conditions become:
\begin{align}
& \bar{L} \mathbf{v} - \lambda \bar{Q} \mathbf{v} - \mu \mathbf{v} = 0, \label{eq:KKT:L}\\
& \mathbf{v}^T \mathbf{v} = vol, \label{eq:KKT:v}\\
& \mathbf{v}^T \bar{Q} \mathbf{v} = \alpha, \label{eq:KKT:Q}\\
& \lambda > 0, \label{eq:KKT:lam}.
\end{align}
Under the assumption that $\alpha$ is arbitrarily assigned by user and $\lambda$ and $\mu$ are independent variables, Eq.(\ref{eq:KKT:L}-\ref{eq:KKT:lam}) cannot be solved explicitly, and it may produce infinite number of feasible solutions, if one exists. As a workaround, we temporarily drop the assumption that $\alpha$ is an arbitrary value assigned by the user. Instead, we assume $\alpha \triangleq \mathbf{v}^T \bar{Q} \mathbf{v}$, i.e. $\alpha$ is defined as such that Eq.(\ref{eq:KKT:Q}) holds. Then we introduce an auxiliary variable, $\beta$, which is defined as the ratio between $\mu$ and $\lambda$:
\begin{equation}\label{eq:setbeta}
\beta \triangleq - \frac{\mu}{\lambda}  vol.
\end{equation}

Now we substitute Eq.(\ref{eq:setbeta}) into Eq.(\ref{eq:KKT:L}) we obtain:
\begin{displaymath}
\bar{L} \mathbf{v} - \lambda \bar{Q} \mathbf{v} + \frac{\lambda \beta}{vol} \mathbf{v} = 0,
\end{displaymath}
or equivalently:
\begin{equation}\label{eq:gen}
\bar{L} \mathbf{v} = \lambda (\bar{Q} - \frac{\beta}{vol} I) \mathbf{v}
\end{equation}
We immediately notice that Eq.(\ref{eq:gen}) is a generalized eigenvalue problem for a given $\beta$.

Next we show that the substitution of $\alpha$ with $\beta$ does not compromise our original intention of lower bounding $\mathbf{v}^T \bar{Q} \mathbf{v}$ in Eq.(\ref{eq:objective}).
\begin{lem}\label{lemma}
$\beta < \mathbf{v}^T \bar{Q} \mathbf{v}$.
\end{lem}
\begin{proof}
Let $\gamma \triangleq \mathbf{v}^T \bar{L} \mathbf{v}$, by left-hand multiplying $\mathbf{v}^T$ to both sides of Eq.(\ref{eq:gen}) we have
\begin{displaymath}
\mathbf{v}^T \bar{L} \mathbf{v} = \lambda \mathbf{v}^T (\bar{Q} - \frac{\beta}{vol} I) \mathbf{v}.
\end{displaymath}
Then incorporating Eq.(\ref{eq:KKT:v}) and $\alpha \triangleq \mathbf{v}^T \bar{Q} \mathbf{v}$ we have
\begin{displaymath}
\gamma = \lambda (\alpha - \beta).
\end{displaymath}
Now recall that $L$ is positive semi-definite (Property~\ref{prop:laplacian}), and so is $\bar{L}$, which means
\begin{displaymath}
\gamma = \mathbf{v}^T \bar{L} \mathbf{v} > 0,\mbox{~}\forall \mathbf{v} \neq D^{1/2} \mathbf{1}.
\end{displaymath}
Consequently, we have
\begin{displaymath}
\alpha - \beta = \frac{\gamma}{\lambda} > 0 \mbox{~~} \Rightarrow \mbox{~~} \mathbf{v}^T \bar{Q} \mathbf{v} = \alpha > \beta.
\end{displaymath}
\end{proof}

In summary, our algorithm works as follows (the exact implementation is shown in Algorithm~\ref{alg}):
\begin{enumerate}
\item \textbf{Generating candidates}: The user specifies a value for $\beta$, and we solve the generalized eigenvalue system given in Eq.(\ref{eq:gen}). Note that both $\bar{L}$ and $\bar{Q} - \beta/vol I$ are Hermitian matrices, thus the generalized eigenvalues are guaranteed to be real numbers.
\item \textbf{Finding the feasible set}: Removing generalized eigenvectors associated with non-positive eigenvalues, and normalize the rest such that $\mathbf{v}^T \mathbf{v} = vol$. Note that the trivial solution $D^{1/2} \mathbf{1}$ is automatically removed in this step because if it is a generalized eigenvector in Eq.(\ref{eq:gen}), the associated eigenvalue would be $0$. Since we have at most $N$ generalized eigenvectors, the number of feasible eigenvectors is at most $N$.
\item \textbf{Choosing the optimal solution}:  We choose from the feasible solutions the one that minimizes $\mathbf{v}^T \bar{L} \mathbf{v}$, say $\mathbf{v}^*$.
\end{enumerate}

According to Lemma~\ref{lemma}, $\mathbf{v}^*$ is the optimal solution to the objective function in Eq.(\ref{eq:objective}) for any given $\beta$ and $\beta < \alpha = \mathbf{v}^{*T} \bar{Q} \mathbf{v}^*$.

\begin{algorithm}[t]\label{alg}
\caption{Constrained Spectral Clustering}
\KwIn{Affinity matrix $A$, constraint matrix $Q$, $\beta$\;}
\KwOut{The optimal (relaxed) cluster indicator $\mathbf{u}^*$\;}
$vol \leftarrow \sum_{i=1}^N \sum_{j=1}^N A_{ij}$, $D \leftarrow \text{diag}(\sum_{j=1}^N A_{ij})$\;
$\bar{L} \leftarrow I - D^{-1/2} A D^{-1/2}$, $\bar{Q} \leftarrow D^{-1/2} Q D^{-1/2}$\;
$\lambda_{max}(\bar{Q}) \leftarrow$ the largest eigenvalue of $\bar{Q}$\;
\If{$\beta \geq \lambda_{max}(\bar{Q}) \cdot vol$}{\Return $\mathbf{u}^*=\emptyset$\;}
\Else{
Solve the generalized eigenvalue system in Eq.(\ref{eq:gen})\;
Remove eigenvectors associated with non-positive eigenvalues and normalize the rest by $\mathbf{v} \leftarrow \frac{\mathbf{v}}{\|\mathbf{v}\|}\sqrt{vol}$\;
$\mathbf{v}^* \leftarrow \argmin_{\mathbf{v}} \mathbf{v}^T \bar{L} \mathbf{v}$, where $\mathbf{v}$ is among the feasible eigenvectors generated in the previous step\;
\Return $\mathbf{u}^* \leftarrow D^{-1/2} \mathbf{v}^*$\;
}
\end{algorithm}

\subsection{A Sufficient Condition for the Existence of Solutions}
\label{sec:model:parameter}

On one hand, our method described above is guaranteed to generate a finite number of feasible solutions. On the other hand, we need to set $\beta$ appropriately so that the generalized eigenvalue system in Eq.(\ref{eq:gen}) combined with the KKT conditions in Eq.(\ref{eq:KKT:L}-\ref{eq:KKT:lam}) will give rise to at least one feasible solution. In this section, we discuss such a sufficient condition:
\begin{displaymath}
\beta < \lambda_{max}(\bar{Q}) \cdot vol,
\end{displaymath}
where $\lambda_{max}(\bar{Q})$ is the largest eigenvalue of $\bar{Q}$. In this case, the matrix on the right hand side of Eq.(\ref{eq:gen}), namely $\bar{Q} - \beta / vol \cdot I$, will have at least one positive eigenvalue. Consequently, the generalized eigenvalue system in Eq.(\ref{eq:gen}) will have at least one positive eigenvalue. Moreover, the number of feasible eigenvectors will increase if we make $\beta$ smaller. For example, if we set $\beta < \lambda_{min}(\bar{Q}) vol$, $\lambda_{min}(\bar{Q})$ to be the smallest eigenvalue of $\bar{Q}$, then $\bar{Q} - \beta / vol \cdot I$ becomes positive definite. Then the generalized eigenvalue system in Eq.(\ref{eq:gen}) will generate $N-1$ feasible eigenvectors (the trivial solution $D^{1/2} \mathbf{1}$ with eigenvalue $0$ is dropped).

In practice, we normally choose the value of $\beta$ within the range
\begin{displaymath}
(\lambda_{min}(\bar{Q}) \cdot vol, \lambda_{max}(\bar{Q}) \cdot vol).
\end{displaymath}
In that range, the greater $\beta$ is, the more the solution will be biased towards satisfying $\bar{Q}$. Again, note that whenever we have $\beta < \lambda_{max}(\bar{Q}) \cdot vol$, the value of $\alpha$ will always be bounded by
\begin{displaymath}
\beta < \alpha \leq \lambda_{max} vol.
\end{displaymath}
Therefore we do not need to take care of $\alpha$ explicitly.

\subsection{An Illustrative Example}
\label{sec:model:example}

\begin{figure}
\centering
\includegraphics*[width=0.4\linewidth]{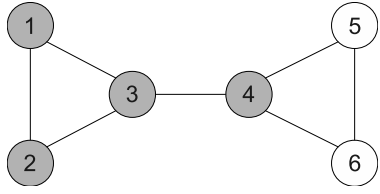}
\caption{An illustrative example: the affinity structure says $\{1,2,3\}$ and $\{4,5,6\}$ while the node labeling (coloring) says $\{1,2,3,4\}$ and $\{5,6\}$.}\label{fig:example}
\end{figure}

To illustrate how our algorithm works, we present a toy example as follows. In Fig.~\ref{fig:example}, we have a graph associated with the following affinity matrix:
\begin{displaymath}
  A=
  \begin{bmatrix}
    0 & 1 & 1 & 0 & 0 & 0\\
    1 & 0 & 1 & 0 & 0 & 0\\
    1 & 1 & 0 & 1 & 0 & 0\\
    0 & 0 & 1 & 0 & 1 & 1\\
    0 & 0 & 0 & 1 & 0 & 1\\
    0 & 0 & 0 & 1 & 1 & 0\\
  \end{bmatrix}
\end{displaymath}
Unconstrained spectral clustering will cut the graph at edge $(3,4)$ and split it into two symmetric parts $\{1,2,3\}$ and $\{4,5,6\}$ (Fig.~\ref{fig:eg_sol:no_cons}).

Then we introduce constraints as encoded in the following constraint matrix:
\begin{displaymath}
  Q=
  \begin{bmatrix}
    +1 & +1 & +1 & +1 & -1 & -1\\
    +1 & +1 & +1 & +1 & -1 & -1\\
    +1 & +1 & +1 & +1 & -1 & -1\\
    +1 & +1 & +1 & +1 & -1 & -1\\
    -1 & -1 & -1 & -1 & +1 & +1\\
    -1 & -1 & -1 & -1 & +1 & +1\\
  \end{bmatrix}.
\end{displaymath}
$Q$ is essentially saying that we want to group nodes $\{1,2,3,4\}$ into one cluster and $\{5,6\}$ the other. Although this kind of ``complete information'' constraint matrix does not happen in practice, we use it here only to make the result more explicit and intuitive.

$\bar{Q}$ has two distinct eigenvalues: $0$ and $2.6667$. As analyzed above, $\beta$ must be smaller than $2.6667 vol$ to guarantee the existence of a feasible solution, and larger $\beta$ means we want more constraints in $Q$ to be satisfied (in a relaxed sense). Thus we set $\beta$ to $vol$ and $2 vol$ respectively, and see how it will affect the resultant constrained cuts. We solve the generalized eigenvalue system in Eq.(\ref{eq:gen}), and plot the cluster indicator vector $\mathbf{u}^*$ in Fig.~\ref{fig:eg_sol:beta_1} and \ref{fig:eg_sol:beta_2}, respectively. We can see that as $\beta$ increases, node $4$ is dragged from the group of nodes $\{5,6\}$ to the group of nodes $\{1,2,3\}$, which conforms to our expectation that greater $\beta$ value implies better constraint satisfaction.

\begin{figure*}
\centering
\subfigure[Unconstrained]{\includegraphics*[width=0.3\linewidth]{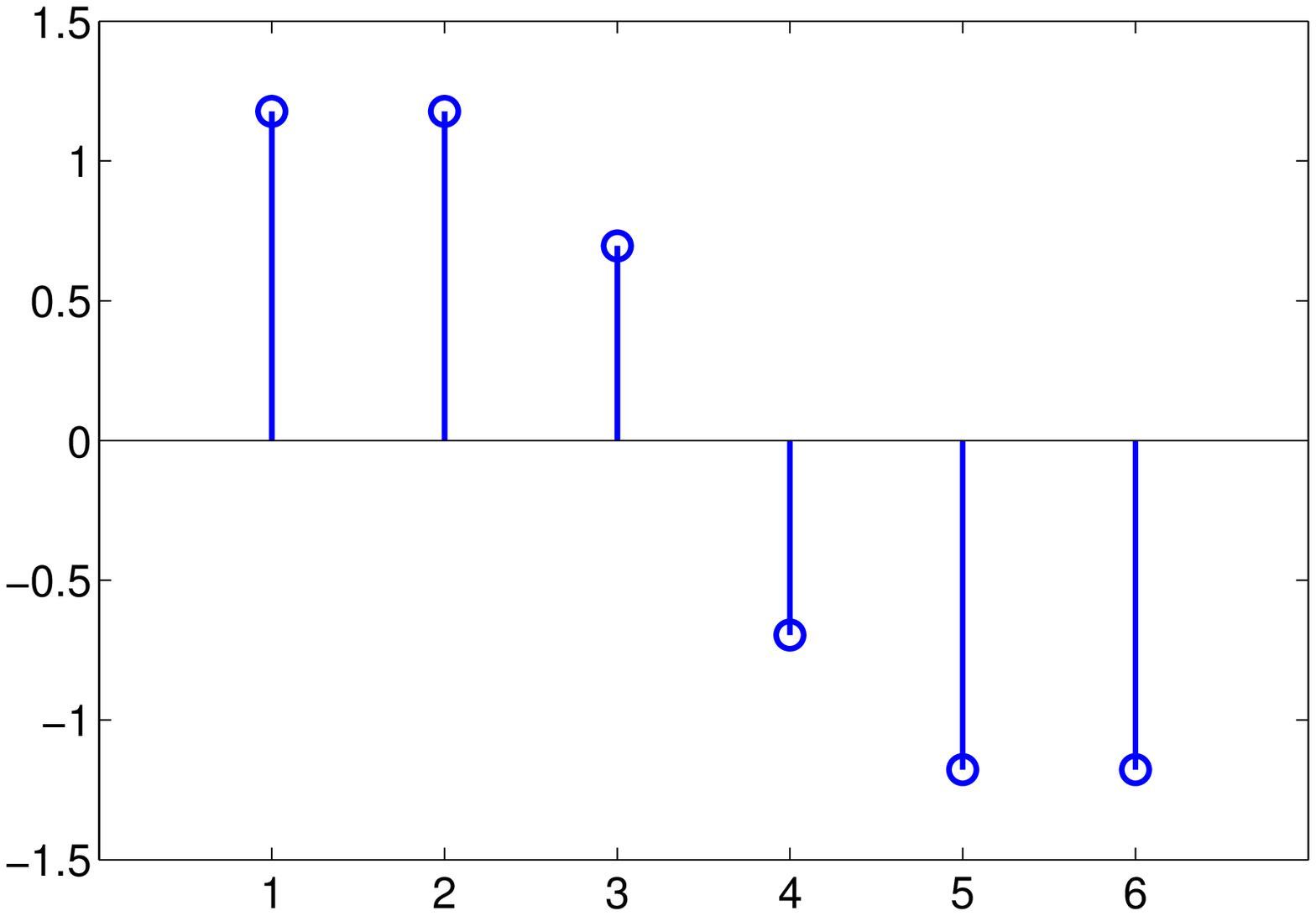}\label{fig:eg_sol:no_cons}}
\subfigure[Constrained, $\beta=vol$]{\includegraphics*[width=0.3\linewidth]{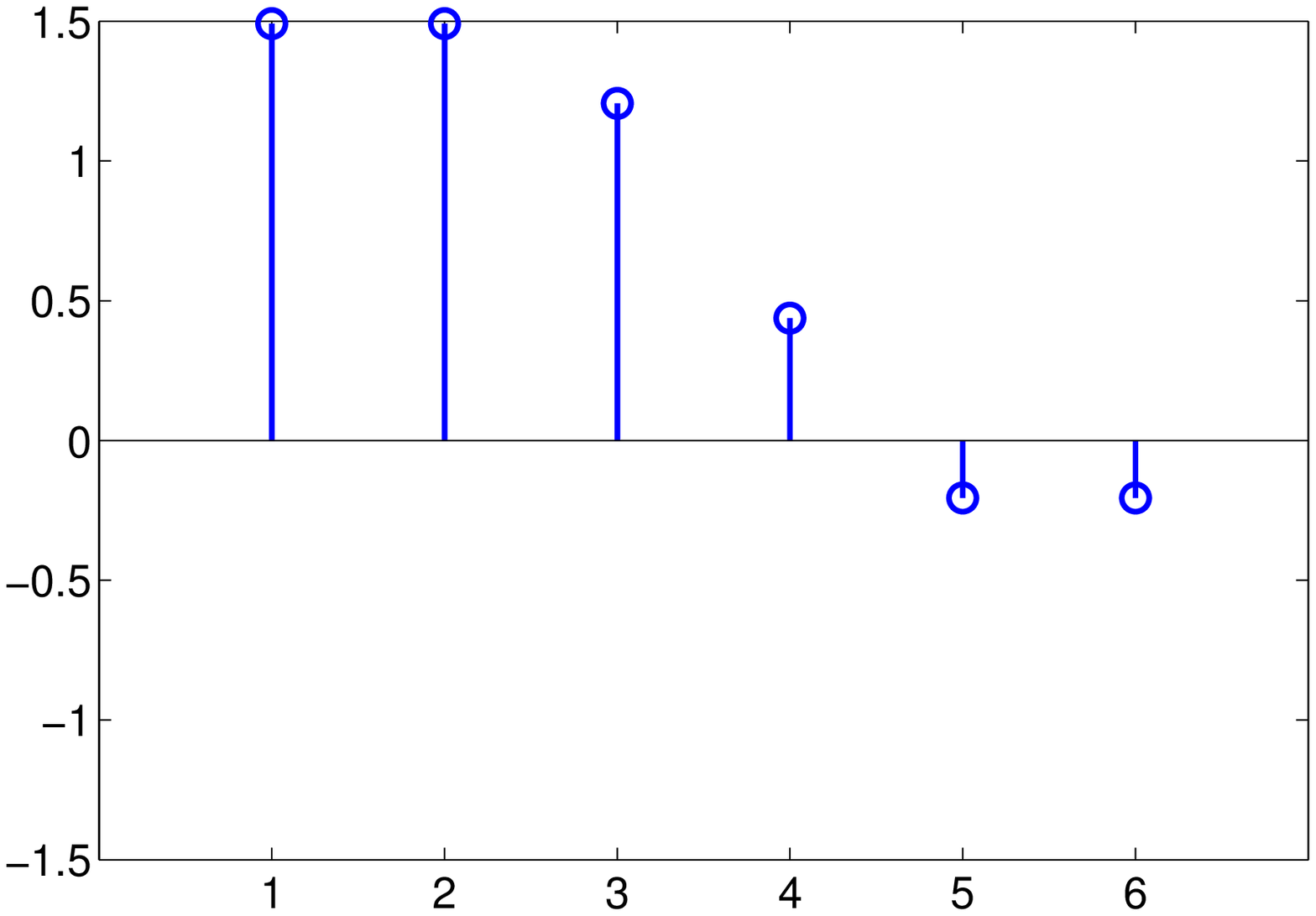}\label{fig:eg_sol:beta_1}}
\subfigure[Constrained, $\beta=2 vol$]{\includegraphics*[width=0.3\linewidth]{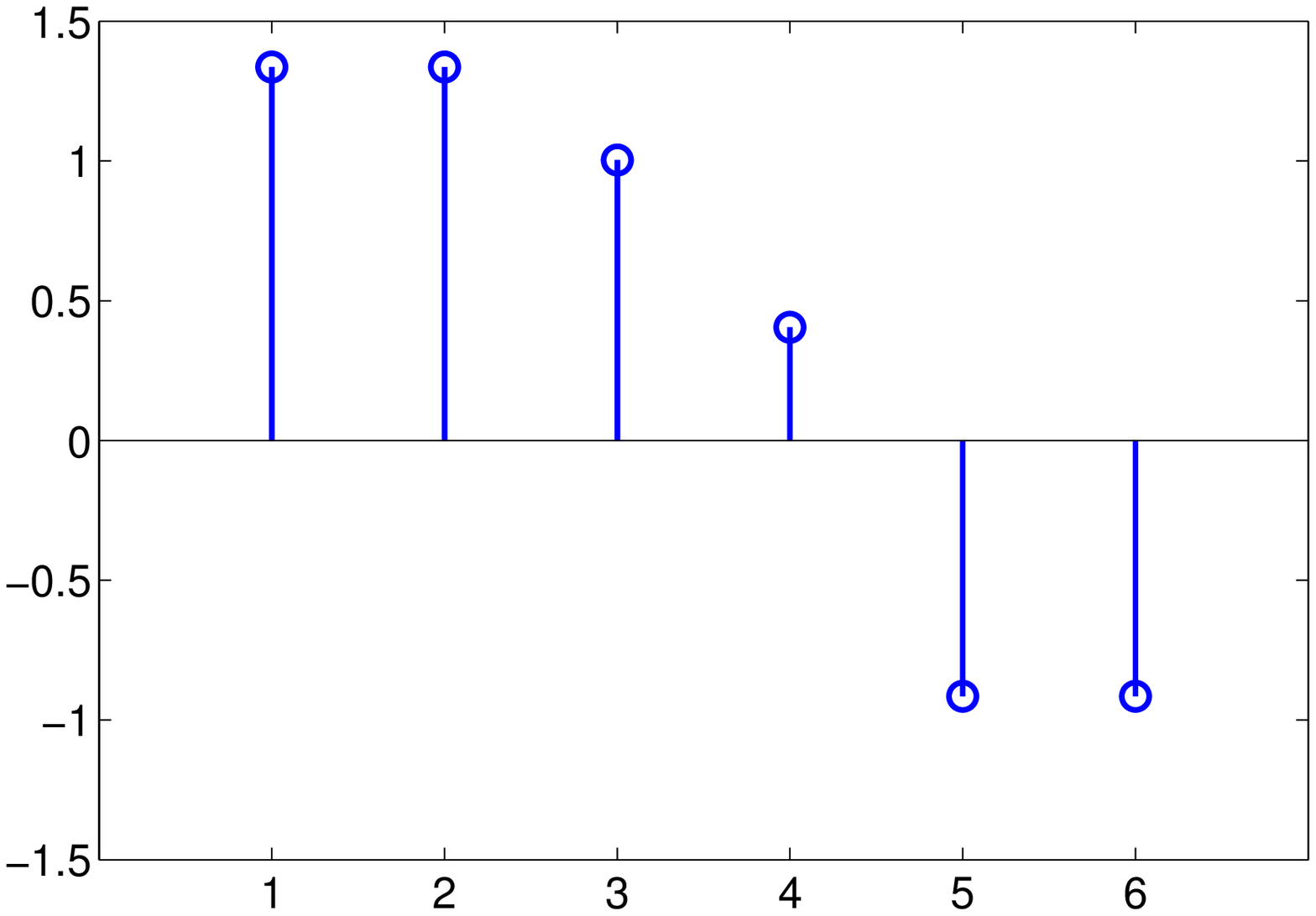}\label{fig:eg_sol:beta_2}}
\caption{The solutions to the illustrative example in Fig.~\ref{fig:example} with different $\beta$. The x-axis is the indices of the instances and the y-axis is the corresponding entry values in the optimal (relaxed) cluster indicator $\mathbf{u}^*$. Notice that node $4$ is biased toward nodes $\{1,2,3\}$ as $\beta$ increases.}\label{fig:eg_sol}
\end{figure*}

\section{Interpretations of Our Formulation}
\label{sec:interpret}

\subsection{A Graph Cut Interpretation}

Unconstrained spectral clustering can be interpreted as finding the Ncut of an \emph{unlabeled} graph. Similarly, our formulation of constrained spectral clustering in Eq.(\ref{eq:objective}) can be interpreted as finding the Ncut of a \emph{labeled/colored} graph.

Specifically, suppose we have an undirected weighted graph. The nodes of the graph are colored in such a way that nodes of the same color are advised to be assigned into the same cluster while nodes of different colors are advised to be assigned into different clusters (e.g. Fig.~\ref{fig:example}). Let $\mathbf{v}^*$ be the solution to the constrained optimization problem in Eq.(\ref{eq:objective}). We cut the graph into two parts based on the values of the entries of $\mathbf{u}^* = D^{-1/2} \mathbf{v}^*$. Then $\mathbf{v}^{*T} \bar{L} \mathbf{v}^*$ can be interpreted as the \textbf{cost} of the cut (in a relaxed sense), which we minimize. On the other hand,
\begin{displaymath}
\alpha = \mathbf{v}^{*T} \bar{Q} \mathbf{v}^* =  \mathbf{u}^{*T} Q \mathbf{u}^*
\end{displaymath}
can be interpreted as the \textbf{purity} of the cut (also in a relaxed sense), according to the color of the nodes in respective sides. For example, if $Q \in \{-1,0,1\}^{N \times N}$ and $\mathbf{u}^* \in \{-1,1\}^N$, then $\alpha$ equals to the number of constraints in $Q$ that are satisfied by $\mathbf{u}^*$ minus the number of constraints violated. More generally, if $Q_{ij}$ is a positive number, then $\mathbf{u}^*_i$ and $\mathbf{u}^*_j$ having the same sign will contribute positively to the purity of the cut, whereas different signs will contribute negatively to the purity of the cut. It is not difficult to see that the purity can be maximized when there is no pair of nodes with different colors that are assigned to the same side of the cut (0 violations), which is the case where all constraints in $Q$ are satisfied.

\subsection{A Geometric Interpretation}
\label{sec:interpret:geo}

We can also interpret our formulation as constraining the joint numerical range (\citet{Horn1990}) of the graph Laplacian and the constraint matrix. Specifically, we consider the joint numerical range:
\begin{equation}\label{eq:jnr}
J(\bar{L},\bar{Q}) \triangleq \{(\mathbf{v}^T \bar{L} \mathbf{v}, \mathbf{v}^T \bar{Q} \mathbf{v}): \mathbf{v}^T \mathbf{v} = 1\}.
\end{equation}
$J(\bar{L},\bar{Q})$ essentially maps all possible cuts $\mathbf{v}$ to a 2-D plane, where the $x$-coordinate corresponds to the cost of the cut, and the $y$-axis corresponds to the constraint satisfaction of the cut. According to our objective in Eq.(\ref{eq:objective}), we want to minimize the first term while lower-bounding the second term. Therefore, we are looking for the leftmost point among those that are above the horizontal line $y = \alpha$.

In Fig.~\ref{fig:jnr}(c), we visualize $J(\bar{L},\bar{Q})$ by plotting all the unconstrained cuts given by spectral clustering and all the constrained cuts given by our algorithm in the joint numerical range, based on the graph Laplacian of a Two-Moon dataset with a randomly generated constraint matrix. The horizontal line and the arrow indicate the constrained area from which we can select feasible solutions. We can see that most of the unconstrained cuts proposed by spectral clustering are far below the threshold, which suggests spectral clustering cannot lead to the ground truth partition (as shown in Fig.~\ref{fig:jnr}(b)) without the help of constraints.

\begin{figure*}
\centering
\subfigure[The unconstrained Ncut]{\includegraphics*[width=0.45\linewidth]{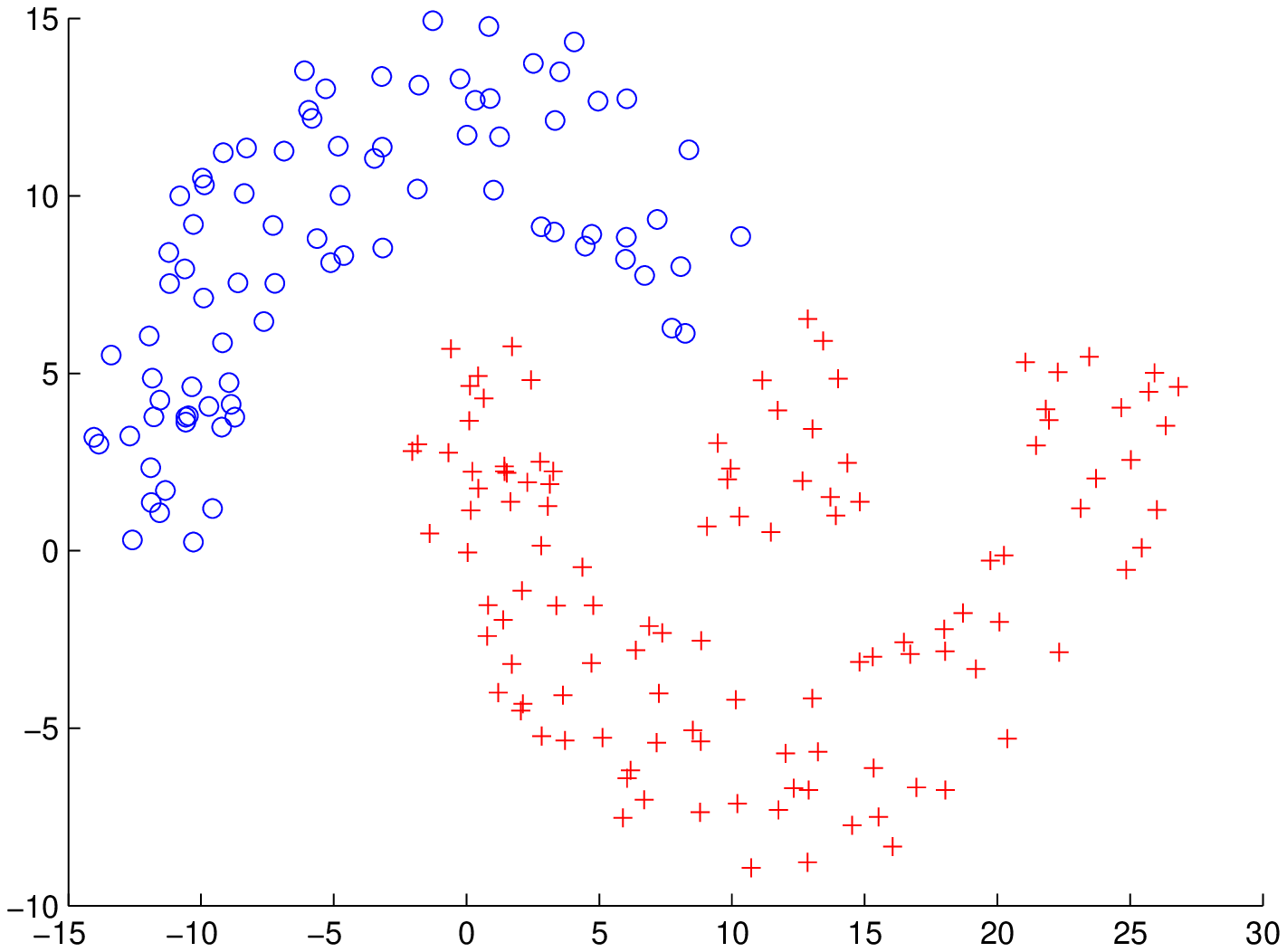}} \quad
\subfigure[The constrained Ncut]{\includegraphics*[width=0.45\linewidth]{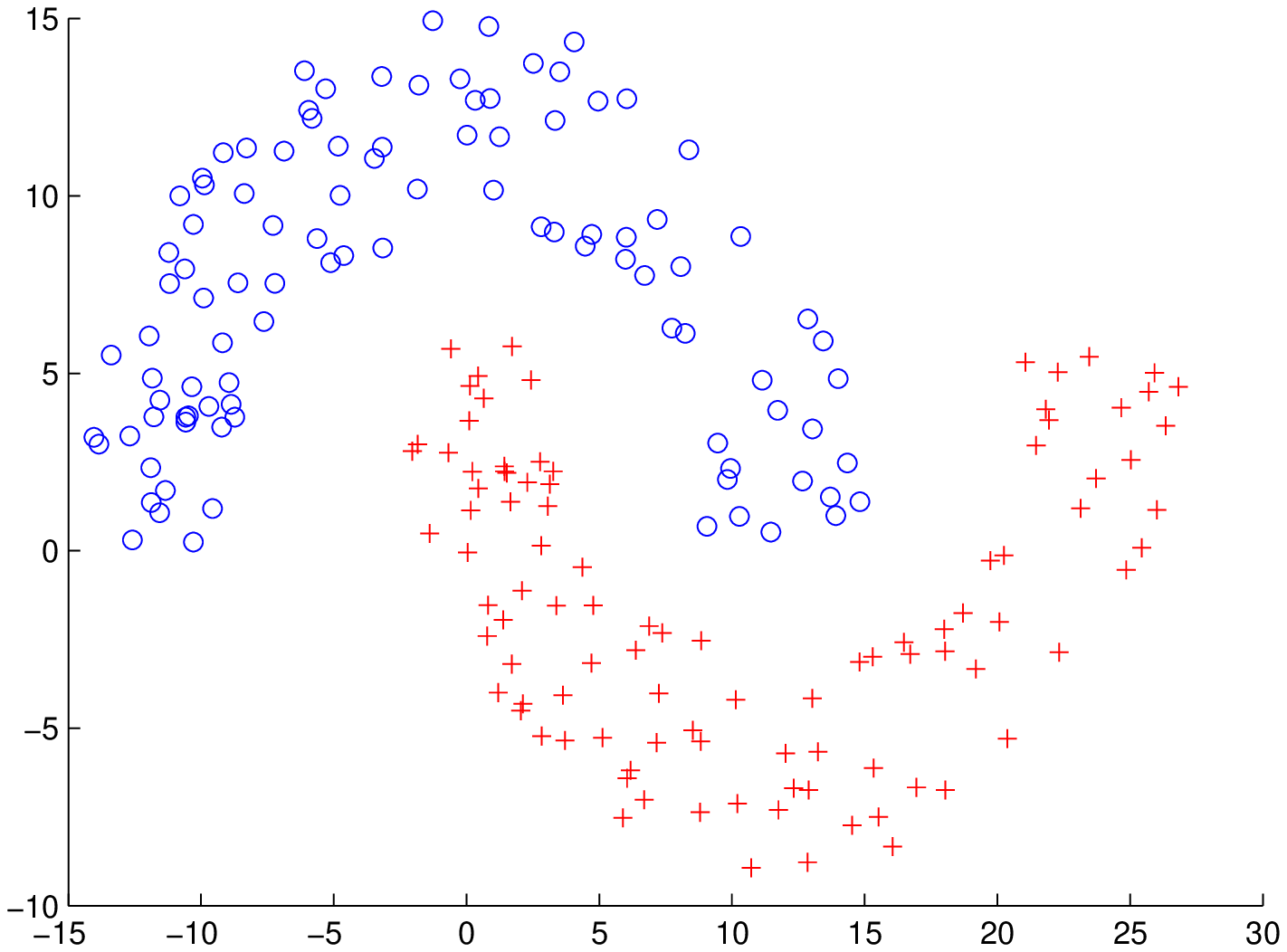}}\\
\subfigure[$J(\bar{L},\bar{Q})$]{\includegraphics*[width=0.8\linewidth]{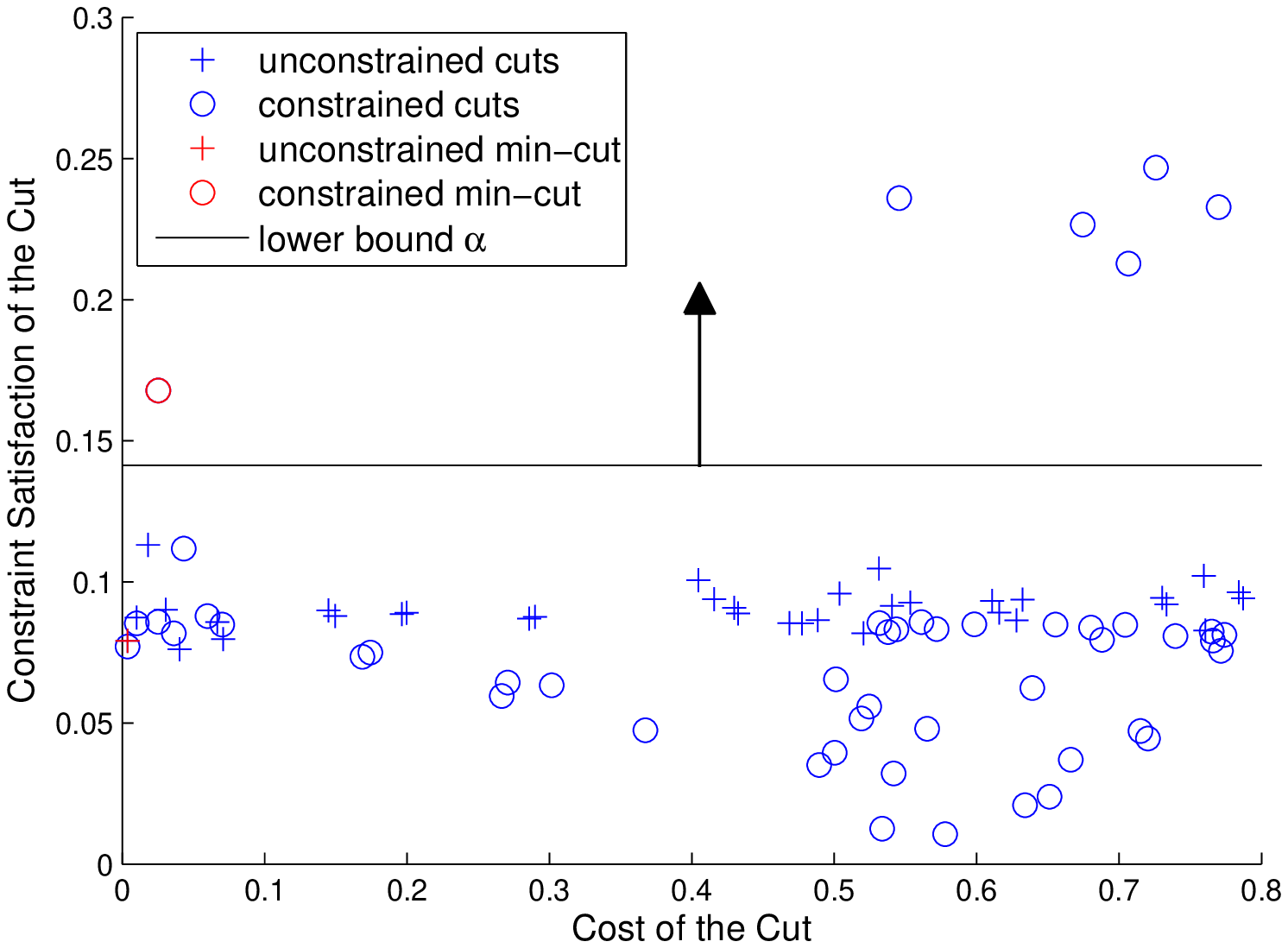}}
\caption{The joint numerical range of the normalized graph Laplacian $\bar{L}$ and the normalized constraint matrix $\bar{Q}$, as well as the optimal solutions to unconstrained/constrained spectral clustering.}\label{fig:jnr}
\end{figure*}

\section{Implementation and Extensions}
\label{sec:algorithm}

In this section, we discuss some implementation issues of our method. Then we show how to generalize it to $K$-way partition where $K \geq 2$.

\subsection{Constrained Spectral Clustering for 2-Way Partition}

The routine of our method is similar to that of unconstrained spectral clustering. The input of the algorithm is an affinity matrix $A$, the
constraint matrix $Q$, and a threshold $\beta$. Then we solve the generalized eigenvalue problem in Eq.(\ref{eq:gen}) and find all the feasible generalized eigenvectors. The output is the optimal (relaxed) cluster assignment indicator $\mathbf{u}^*$. In practice, a partition is often derived from $\mathbf{u}^*$ by assigning nodes corresponding to the positive entries in $\mathbf{u}^*$ to one side of the cut, and negative entries to the other side. Our algorithm is summarized in Algorithm~\ref{alg}.

Since our model encodes soft constraints as degree of belief, inconsistent constraints in $Q$ will not corrupt our algorithm. Instead, they are enforced implicitly by maximizing $\mathbf{u}^T Q \mathbf{u}$. Note that if the constraint matrix $Q$ is generated from a partial labeling, then the constraints in $Q$ will always be consistent.

\textbf{Runtime analysis:} The runtime of our algorithm is dominated by that of the generalized eigendecomposition. In other words, the complexity of our algorithm is on a par with that of unconstrained spectral clustering in big-O notation, which is $\mathcal{O}(kN^2)$, $N$ to be the number of data instances and $k$ to be the number of eigenpairs we need to compute. Here $k$ is a number large enough to guarantee the existence of feasible solutions. In practice we normally have $2 < k \ll N$.

\subsection{Extension to $K$-Way Partition}
\label{sec:algorithm:k-way}

Our algorithm can be naturally extended to $K$-way partition for $K>2$, following what we usually do for unconstrained spectral clustering (\citet{Luxburg2007}): instead of only using the optimal feasible eigenvector $\mathbf{u}^*$, we preserve top-$(K-1)$ eigenvectors associated with positive eigenvalues, and perform $K$-means algorithm based on that embedding.

Specifically, the constraint matrix $Q$ follows the same encoding scheme: $Q_{ij}>0$ if node $i$ and $j$ are believed to belong to the same cluster, $Q_{ij}<0$ otherwise. To guarantee we can find $K-1$ feasible eigenvectors, we set the threshold $\beta$ such that
\begin{displaymath}
\beta < \lambda_{K-1} vol,
\end{displaymath}
where $\lambda_{K-1}$ is the $(K-1)$-th largest eigenvalue of $\bar{Q}$. Given all the feasible eigenvectors, we pick the top $K-1$ in terms of minimizing $\mathbf{v}^T \bar{L} \mathbf{v}$ \footnote{Here we assume the trivial solution, the eigenvector with all 1's, has been excluded.}. Let the $K-1$ eigenvectors form the columns of $V \in \mathbb{R}^{N \times (K-1)}$. We perform $K$-means clustering on the rows of $V$ and get the final clustering. Algorithm~\ref{alg_K} shows the complete routine.

Note that $K$-means is only one of many possible discretization techniques that can derive a $K$-way partition from the relaxed indicator matrix $D^{-1/2} V^*$. Due to the orthogonality of the eigenvectors, they can be independently discretized first, e.g. we can replace Step 11 of Algorithm~\ref{alg_K} with:
\begin{equation}\label{eq:sign}
\mathbf{u}^* \leftarrow \textbf{kmeans}(\textbf{sign}(D^{-1/2} V^*), K).
\end{equation}
This additional step can help alleviate the influence of possible outliers on the $K$-means step in some cases.

Moreover, notice that the feasible eigenvectors, which are the columns of $V^*$, are treated equally in Eq.(\ref{eq:sign}). This may not be ideal in practice because these candidate cuts are not equally favored by graph $\mathcal{G}$, i.e. some of them have higher costs than the other. Therefore, we can weight the columns of $V^*$ with the inverse of their respective costs:
\begin{equation}\label{eq:weight}
\mathbf{u}^* \leftarrow \textbf{kmeans}(\textbf{sign}(D^{-1/2} V^* (V^{*T} \bar{L} V^*)^{-1}), K).
\end{equation}

\begin{algorithm}[t]\label{alg_K}
\caption{Constrained Spectral Clustering for $K$-way Partition}
\KwIn{Affinity matrix $A$, constraint matrix $Q$, $\beta$, $K$\;}
\KwOut{The cluster assignment indicator $\mathbf{u}^*$\;}
$vol \leftarrow \sum_{i=1}^N \sum_{j=1}^N A_{ij}$, $D \leftarrow \text{diag}(\sum_{j=1}^N A_{ij})$\;
$\bar{L} \leftarrow I - D^{-1/2} A D^{-1/2}$, $\bar{Q} \leftarrow D^{-1/2} Q D^{-1/2}$\;
$\lambda_{max} \leftarrow$ the largest eigenvalue of $\bar{Q}$\;
\If{$\beta \geq \lambda_{K-1} vol$}{\Return $\mathbf{u}^*=\emptyset$\;}
\Else{
Solve the generalized eigenvalue system in Eq.(\ref{eq:gen})\;
Remove eigenvectors associated with non-positive eigenvalues and normalize the rest by $\mathbf{v} \leftarrow \frac{\mathbf{v}}{\|\mathbf{v}\|}\sqrt{vol}$\;
$V^* \leftarrow \argmin_{V \in \mathbb{R}^{N \times (K-1)}} \text{trace}(V^T \bar{L} V)$, where the columns of $V$ are a subset of the feasible eigenvectors generated in the previous step\;
\Return $\mathbf{u}^* \leftarrow \textbf{kmeans}(D^{-1/2} V^*, K)$\;
}
\end{algorithm}

\subsection{Using Constrained Spectral Clustering for Transfer Learning}
\label{sec:algorithm:transfer}

The constrained spectral clustering framework naturally fits into the scenario of transfer learning between two graphs. Assume we have two graphs, a source graph and a target graph, which share the same set of nodes but have different sets of edges (or edge weights). The goal is to transfer knowledge from the source graph so that we can improve the cut on the target graph. The knowledge to transfer is derived from the source graph in the form of soft constraints.

Specifically, let $\mathcal{G}_S(V,E_S)$ be the source graph, $\mathcal{G}_T(V,E_T)$ the target graph. $A_S$ and $A_T$ are their respective affinity matrices. Then $A_S$ can be considered as a constraint matrix with only ML constraints. It carries the structural knowledge from the source graph, and we can transfer it to the target graph using our constrained spectral clustering formulation:
\begin{equation}
\argmin_{\mathbf{v} \in \mathbb{R}^N} \mathbf{v}^T \bar{L}_T \mathbf{v},~\text{s.t.}~\mathbf{v}^T A_S \mathbf{v} \geq \alpha,~\mathbf{v}^T \mathbf{v} = vol,~\mathbf{v} \neq D_T^{1/2} \mathbf{1}.
\end{equation}
$\alpha$ is now the lower bound of how much knowledge from the source graph must be enforced on the target graph. To solution to this is similar to Eq.(\ref{eq:gen}):
\begin{equation}
\bar{L}_T \mathbf{v} = \lambda (\bar{A}_S - \frac{\beta}{\text{vol}(\mathcal{G}_T)} I) \mathbf{v}
\end{equation}
Note that since the largest eigenvalue of $\bar{A}_S$ corresponds to a trivial cut, in practice we should set the threshold such that $\beta < \lambda_1 vol$, $\lambda_1$ to be the second largest eigenvalue of $\bar{A}_S$. This will guarantee a feasible eigenvector that is non-trivial.

\section{Testing and Innovative Uses of Our Work}
\label{sec:exp}

We begin with three sets of experiments to test our approach on standard spectral clustering data sets. We then show that since our approach can handle large amounts of soft constraints in a flexible fashion, this opens up two innovative uses of our work: encoding multiple metrics for translated document clustering and transfer learning for fMRI analysis.

We aim to answer the following questions with the empirical study:
\begin{itemize}
\item Can our algorithm effectively incorporate side information and generate semantically meaningful partitions?
\item Does our algorithm converge to the underlying ground truth partition as more constraints are provided?
\item How does our algorithm perform on real-world datasets, as evaluated against ground truth labeling, with comparison to existing techniques?
\item How well does our algorithm handle soft constraints?
\item How well does our algorithm handle large amounts of constraints?
\end{itemize}

Recall that in Section~\ref{sec:intro} we listed four different types of side information: explicit pairwise constraints, partial labeling, alternative metrics, and transfer of knowledge. The empirical results presented in this section are arranged accordingly.

All but one (the fMRI scans) datasets used in our experiments are publicly available online. We implemented our algorithm in MATLAB, which is available online at \url{http://bayou.cs.ucdavis.edu/} or by contacting the authors.

\subsection{Explicit Pairwise Constraints: Image Segmentation}
\label{sec:exp:image}

We demonstrate the effectiveness of our algorithm for image segmentation using explicit pairwise constraints assigned by users.

We choose the image segmentation application for several reasons: 1) it is one of the applications where spectral clustering significantly outperforms other clustering techniques, e.g. $K$-means; 2) the results of image segmentation can be easily interpreted and evaluated by human; 3) instead of generating random constraints, we can provide semantically meaningful constraints to see if the constrained partition conforms to our expectation.

The images we used were chosen from the Berkeley Segmentation Dataset and Benchmark (\citet{Martin2001}). The original images are $480 \times 320$ grayscale images in jpeg format. For efficiency consideration, we compressed them to $10\%$ of the original size, which is $48 \times 32$ pixels, as shown in Fig.~\ref{fig:ele:small} and \ref{fig:hat:small}. Then affinity matrix of the image was computed using the RBF kernel, based on both the positions and the grayscale values of the pixels. As a baseline, we used unconstrained spectral clustering (\citet{Shi2000}) to generate a 2-segmentation of the image. Then we introduced different sets of constraints to see if they can generate expected segmentation. Note that the results of image segmentation vary with the number of segments. To save us from the complications of parameter tuning, which is irrelevant to the contribution of this work, we always set the number of segments to be 2.

\begin{figure*}
\centering
\subfigure[Original image]{\includegraphics*[width=0.36\linewidth]{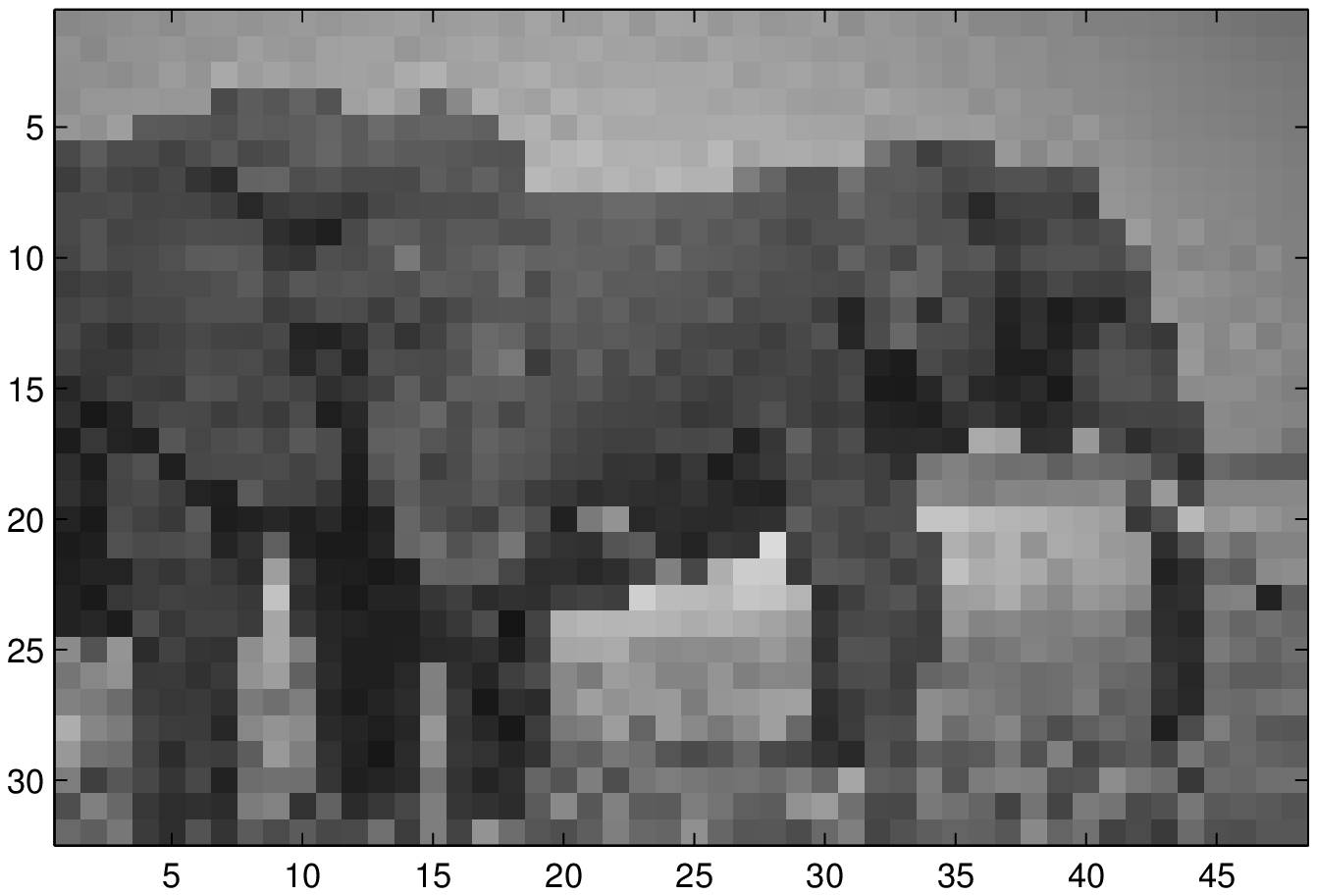}\label{fig:ele:small}}
\subfigure[No constraints]{\includegraphics*[width=0.36\linewidth]{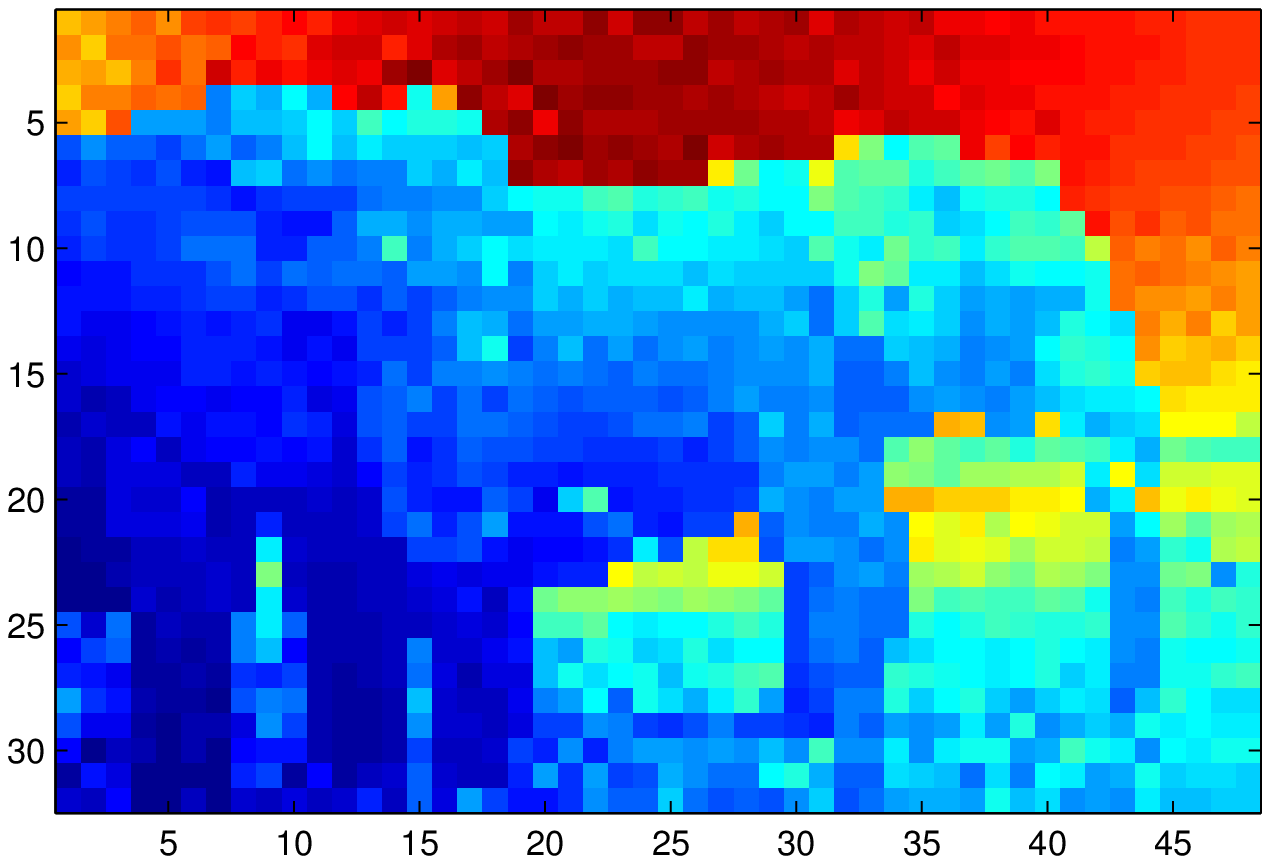}\label{fig:ele:uncon}}\\
\subfigure[Constraint Set 1]{\includegraphics*[width=0.36\linewidth]{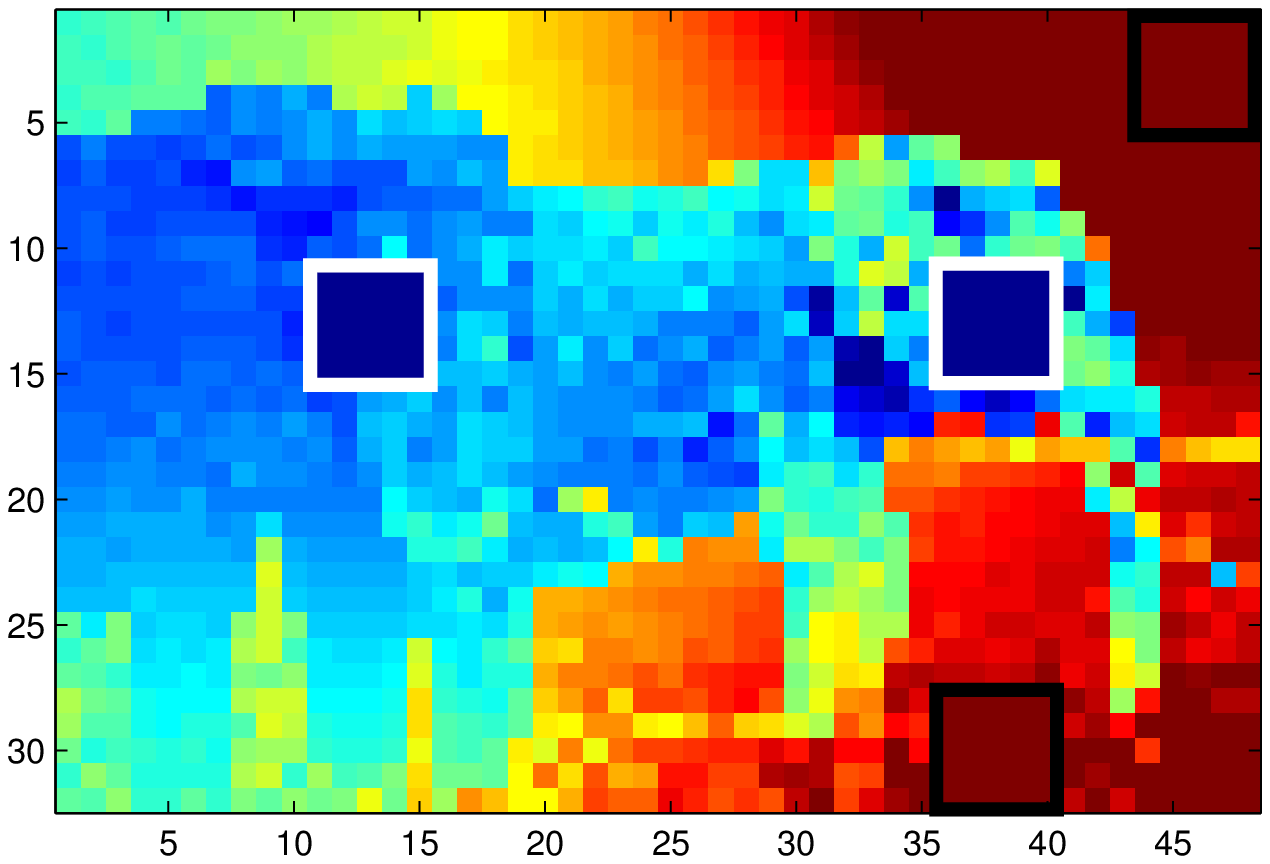}\label{fig:ele:con1}}
\subfigure[Constraint Set 2]{\includegraphics*[width=0.36\linewidth]{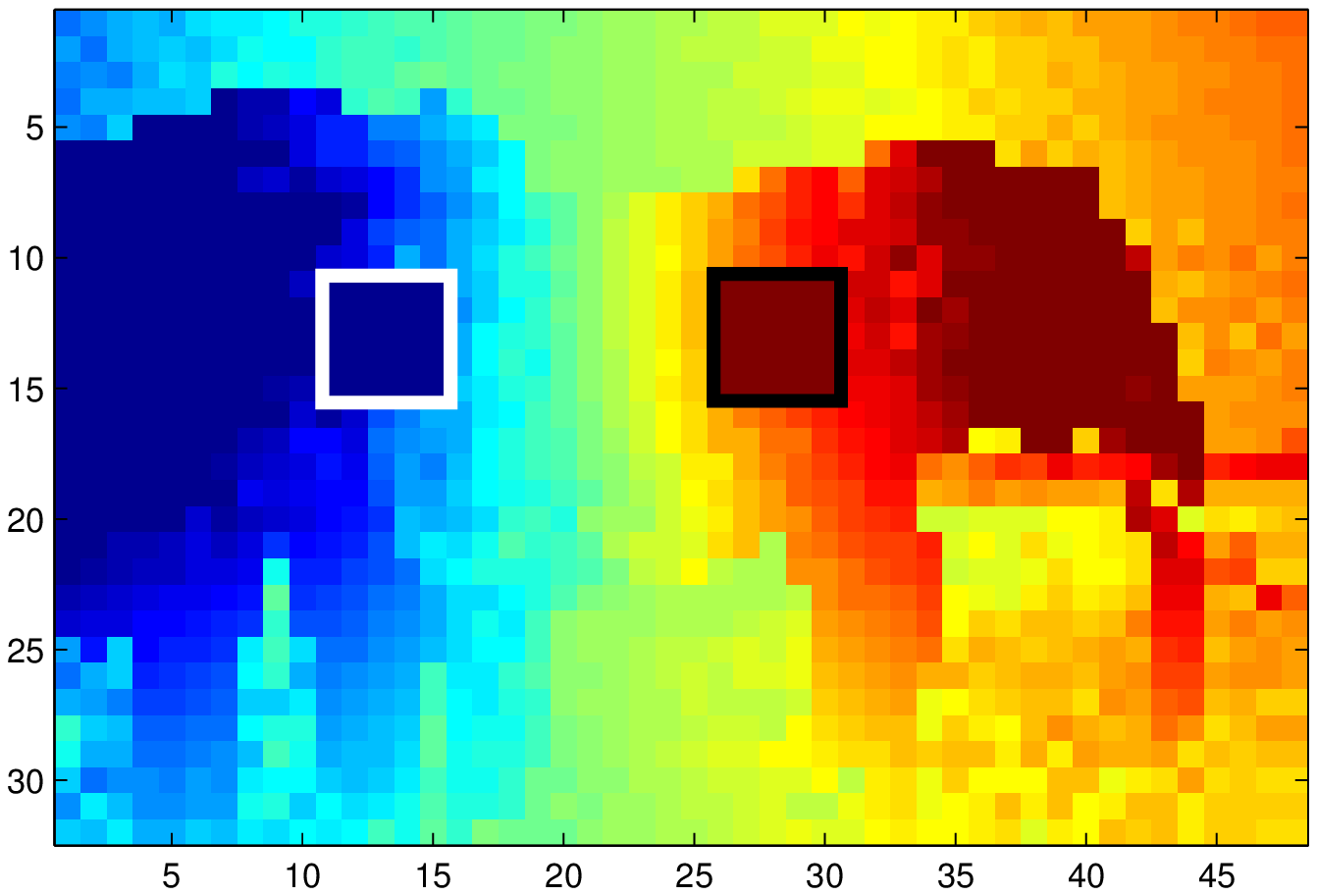}\label{fig:ele:con2}}
\caption{Segmentation of the elephant image. The images are reconstructed based on the relaxed cluster indicator $\mathbf{u}^*$. Pixels that are closer to the red end of the spectrum belong to one segment and blue the other. The labeled pixels are as bounded by the black and white rectangles.}\label{fig:elephant}
\end{figure*}

\begin{figure*}
\centering
\subfigure[Original image]{\includegraphics*[width=0.24\linewidth]{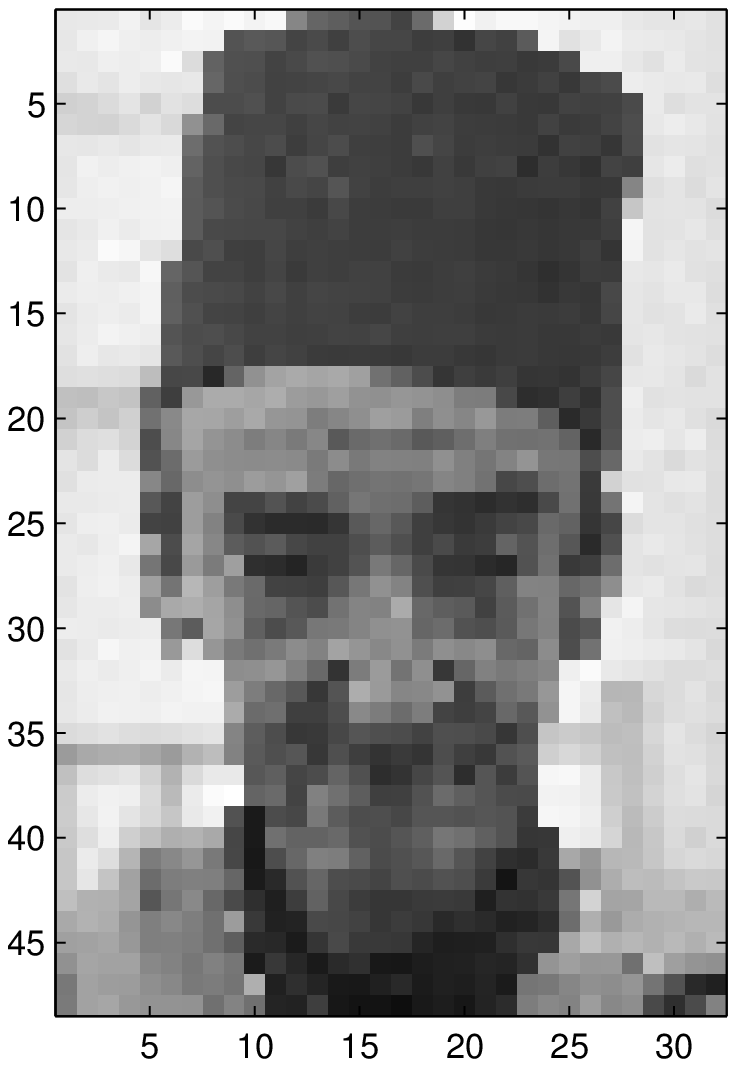}\label{fig:hat:small}}
\subfigure[No constraints]{\includegraphics*[width=0.24\linewidth]{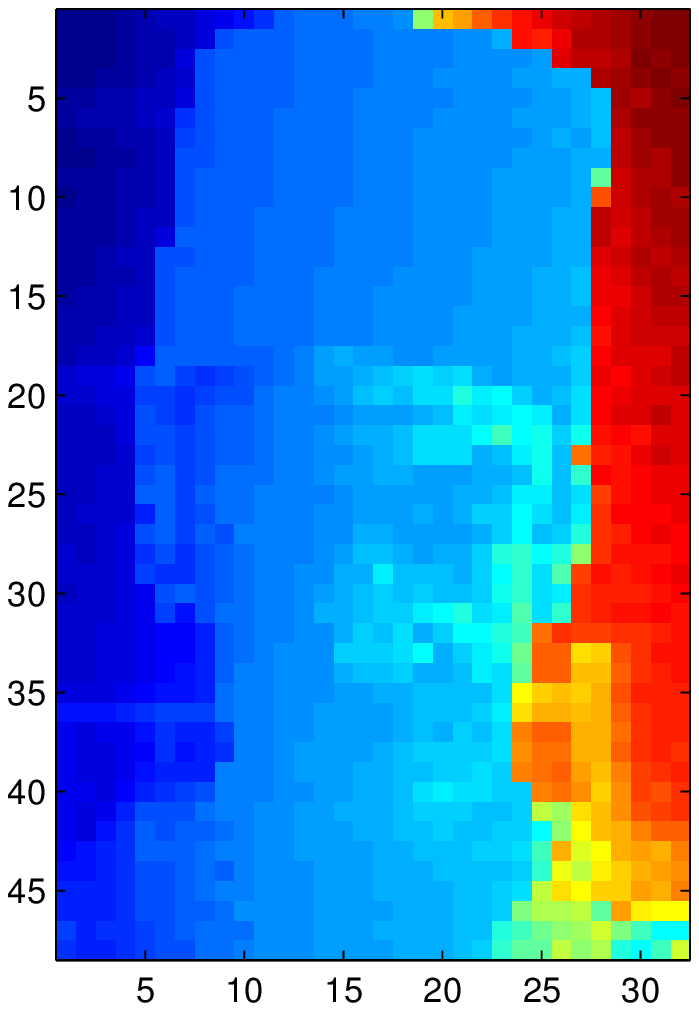}\label{fig:hat:uncon}}
\subfigure[Constraint Set 1]{\includegraphics*[width=0.24\linewidth]{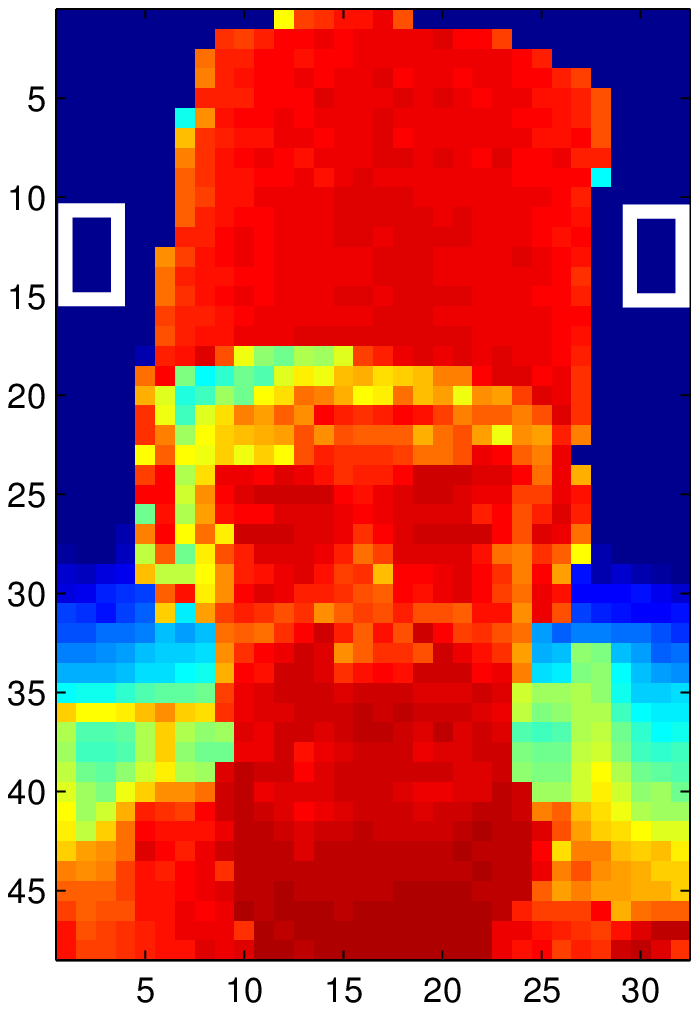}\label{fig:hat:con1}}
\subfigure[Constraint Set 2]{\includegraphics*[width=0.24\linewidth]{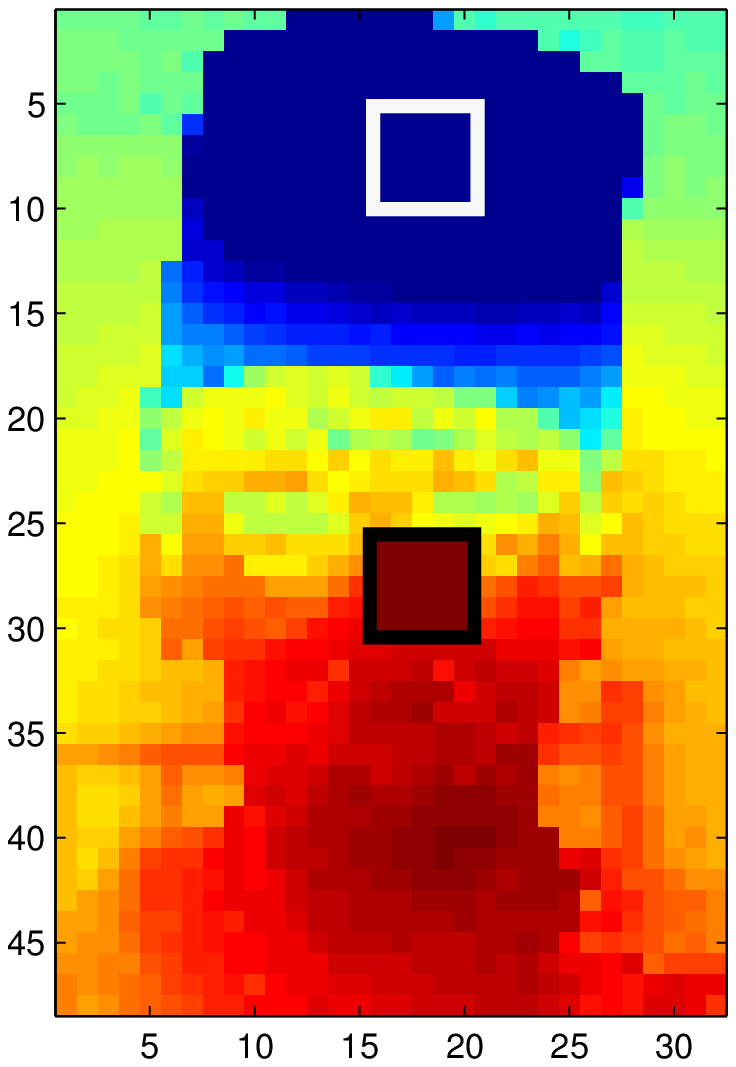}\label{fig:hat:con2}}
\caption{Segmentation of the face image.The images are reconstructed based on the relaxed cluster indicator $\mathbf{u}^*$. Pixels that are closer to the red end of the spectrum belong to one segment and blue the other. The labeled pixels are as bounded by the black and white rectangles.}\label{fig:hat}
\end{figure*}

The results are shown in Fig.~\ref{fig:elephant} and \ref{fig:hat}. To visualize the resultant segmentation, we reconstructed the image using the entry values in the relaxed cluster indicator vector $\mathbf{u}^*$. In Fig.~\ref{fig:ele:uncon}, the unconstrained spectral clustering partitioned the elephant image into two parts: the sky (red pixels) and the two elephants and the ground (blue pixels). This is not satisfying in the sense that it failed to isolate the elephants from the background (the sky and the ground). To correct this, we introduced
constraints by labeling two $5 \times 5$ blocks to be $1$ (as bounded by the black rectangles in Fig.~\ref{fig:ele:con1}): one at the upper-right corner of the image (the sky) and the other at the lower-right corner (the ground); we also labeled two $5 \times 5$ blocks on the heads of the two elephants to be $-1$  (as bounded by the white rectangles in Fig.~\ref{fig:ele:con1}). To generate the constraint matrix $Q$, a ML was added between every pair of pixels with the same label and a CL was added between every pair of pixels with different labels. The parameter $\beta$ was set to
\begin{equation}\label{eq:setbetaexp}
\beta = \lambda_{max} \times  vol \times (0.5 + 0.4 \times \frac{\text{\# of constraints}}{N^2}),
\end{equation}
where $\lambda_{max}$ is the largest eigenvalue of $\bar{Q}$. In this way, $\beta$ is always between $0.5 \lambda_{max} vol$ and $0.9 \lambda_{max} vol$, and it will gradually increase as the number of constraints increases. From Fig.~\ref{fig:ele:con1} we can see that with the help of user supervision, our method successfully isolated the two elephants (blue) from the background, which is the sky and the ground (red). Note that our method achieved this with very simple labeling: four squared blocks.

To show the flexibility of our method, we tried a different set of constraints on the same elephant image with the same parameter settings. This time we aimed to separate the two elephants from each other, which is impossible in the unconstrained case because the two elephants are not only similar in color (grayscale value) but also adjacent in space. Again we used two $5 \times 5$ blocks  (as bounded by the black and white rectangles in Fig.~\ref{fig:ele:con2}), one on the head of the elephant on the left, labeled to be $1$, and the other on the body of the elephant on the right, labeled to be $-1$. As shown in Fig.~\ref{fig:ele:con2}, our method cut the image into two parts with one elephant on the left (blue) and the other on the right (red), just like what a human user would do.

Similarly, we applied our method on a human face image as shown in Fig.~\ref{fig:hat:small}. The unconstrained spectral clustering failed to isolate the human face from the background (Fig.~\ref{fig:hat:uncon}). This is because the tall hat breaks the spatial continuity between the left side of the background and the right side. Then we labeled two $5 \times 3$ blocks to be in the same class, one on each side of the
background. As we intended, our method assigned the background of both sides into the same cluster and thus isolated the human face with his tall hat from the background(Fig.~\ref{fig:hat:con1}). Again, this was achieved simply by labeling two blocks in the image, which covered about $3\%$ of all pixels. Alternatively, if we labeled a $5 \times 5$ block in the hat to be $1$, and a $5 \times 5$ block in the face to be $-1$, the resultant clustering will  isolate the hat from the rest of the image (Fig.~\ref{fig:hat:con2}).

\subsection{Explicit Pairwise Constraints: The Double Moon Dataset}
\label{sec:exp:2moon}

We further examine the behavior of our algorithm on a synthetic dataset using explicit constraints that are derived from underlying ground truth.

We claim that our formulation is a natural extension to spectral clustering. The question to ask then is whether the output of our algorithm converges to that of spectral clustering. More specifically, consider the ground truth partition defined by performing spectral clustering on an ideal distribution. We draw an imperfect sample from the distribution, on which spectral clustering is unable to find the ground truth partition. Then we perform our algorithm on this imperfect sample. As more and more constraints are provided, we want to know whether or not the partition found by our algorithm would converge to the ground truth partition.

To answer the question, we used the Double Moon distribution. As shown in Fig.~\ref{fig:two_moon}, spectral clustering is able to find the two moons when the sample is dense enough. In Fig.~\ref{fig:2moon}(a), we generated an under-sampled instance of the distribution with 100 data points, on which unconstrained spectral clustering could no longer find the ground truth partition. Then we performed our algorithm on this imperfect sample, and compared the partition found by our algorithm to the ground truth partition in terms of adjusted Rand index (ARI, \citet{ARI}). ARI indicates how well a given partition conform to the ground truth: 0 means the given partition is no better than a random assignment; 1 means the given partition matches the ground truth exactly. For each random sample, we generated 50 random sets of constraints and recorded the average ARI. We repeated the process on 10 different random samples of the same size and reported the results in Fig.~\ref{fig:2moon}(b). We can see that our algorithm consistently converge to the ground truth result as more constraints are provided. Notice that there is performance drop when an extreme small number of constraints are provided (less than 10), which is expected because such small number of constraints are insufficient to hint a better partition, and consequentially lead to random perturbation to the results. As more constraints were provided, the results were quickly stabilized.

\begin{figure*}
\centering
\subfigure[A Double Moon sample and its Ncut]{\includegraphics*[width=0.45\linewidth]{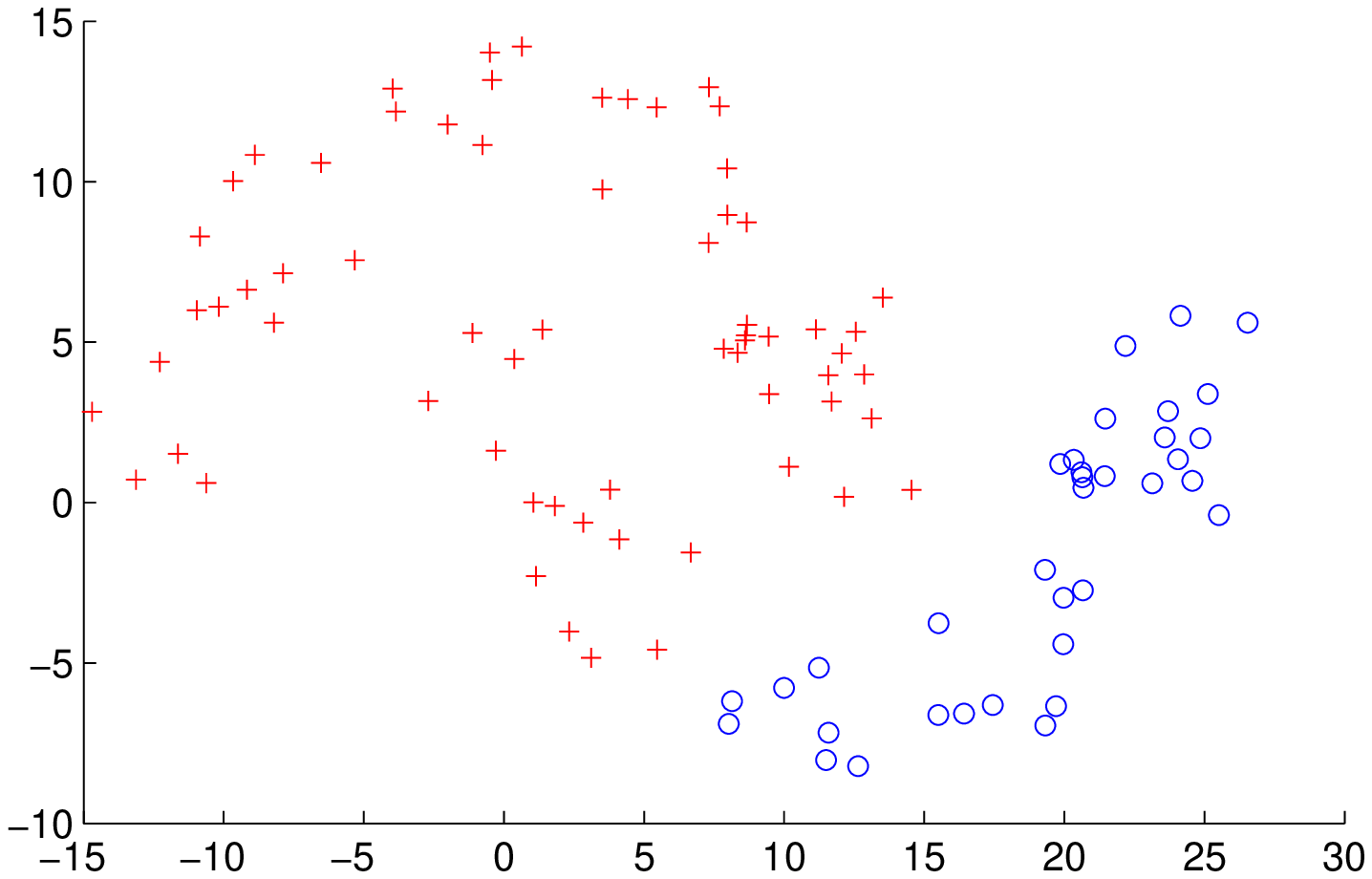}}
\subfigure[The convergence of our algorithm]{\includegraphics*[width=0.45\linewidth]{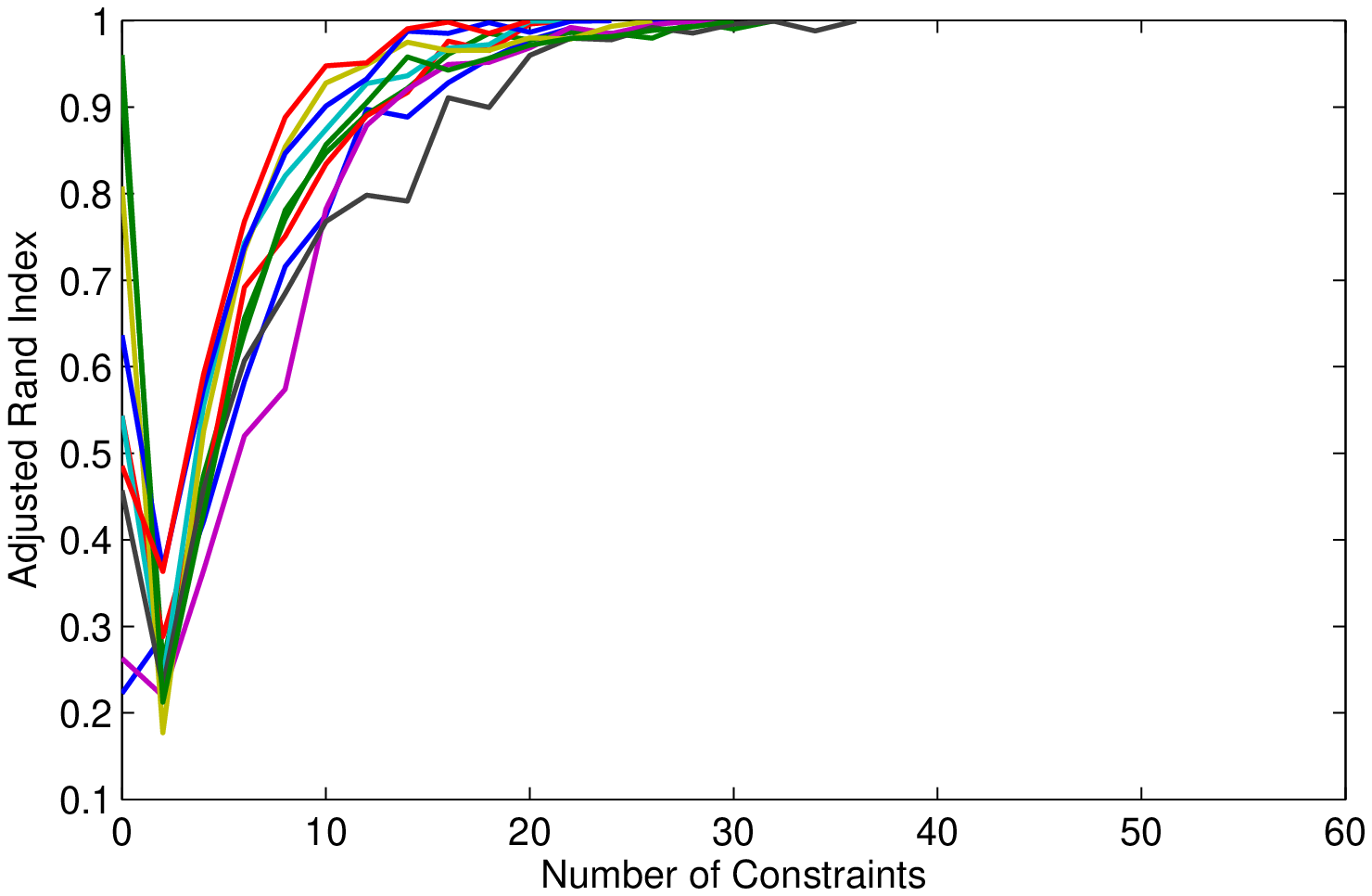}}
\caption{The convergence of our algorithm on 10 random samples of the Double Moon distribution.}\label{fig:2moon}
\end{figure*}

To illustrate the robustness of the our approach, we created a Double Moon sample with uniform background noise, as shown in Fig.~\ref{fig:2moon:noisy}. Although the sample is dense enough (600 data instances in total), spectral clustering fails to find the correctly identify the two moons, due to the influence of background noise (100 data instances). However, with 20 constraints, our algorithm is able to recover the two moons in spite of the background noise.

\begin{figure*}
\centering
\subfigure[Spectral Clustering]{\includegraphics*[width=0.45\linewidth]{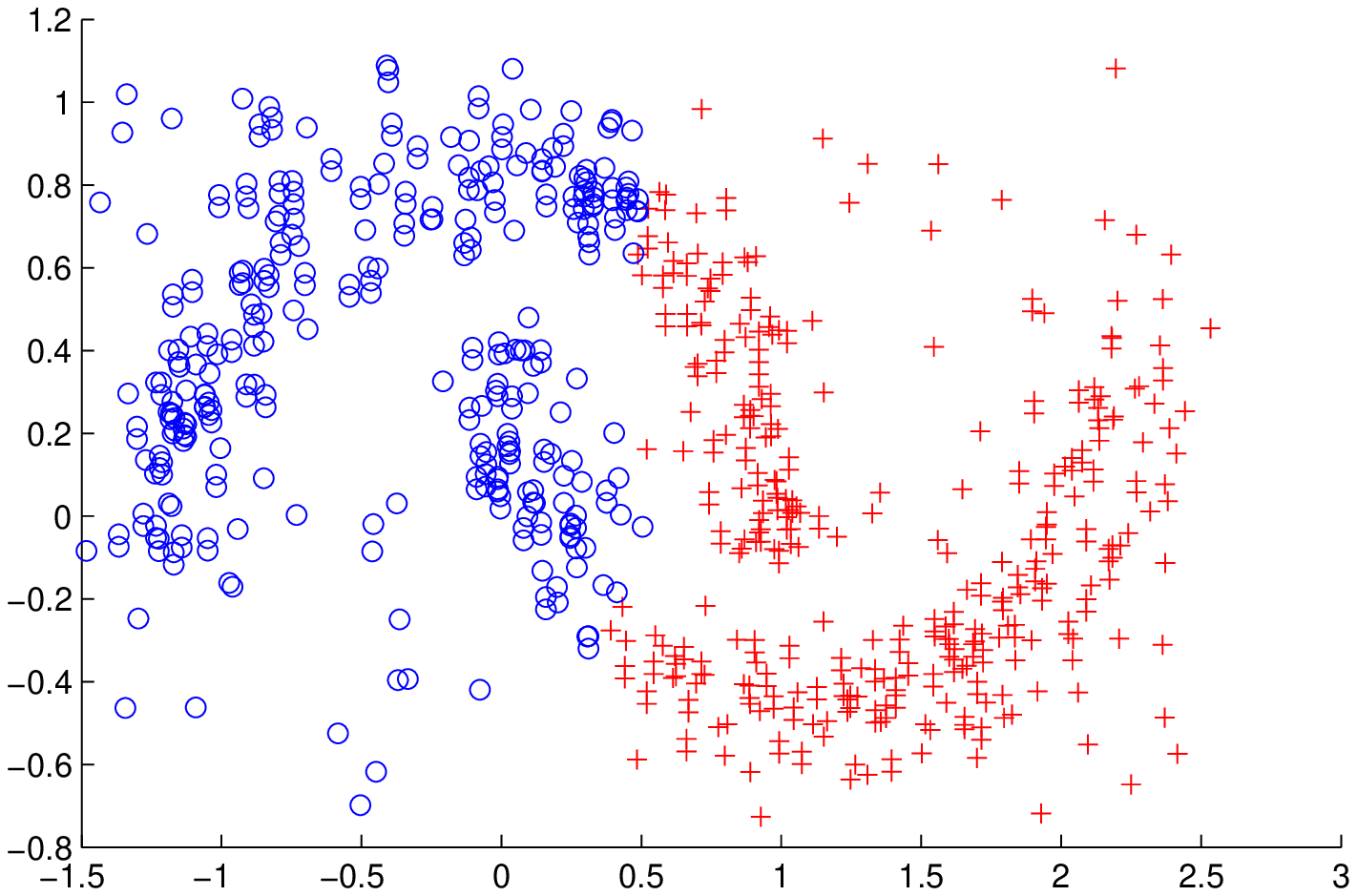}}
\subfigure[Constrained Spectral Clustering]{\includegraphics*[width=0.45\linewidth]{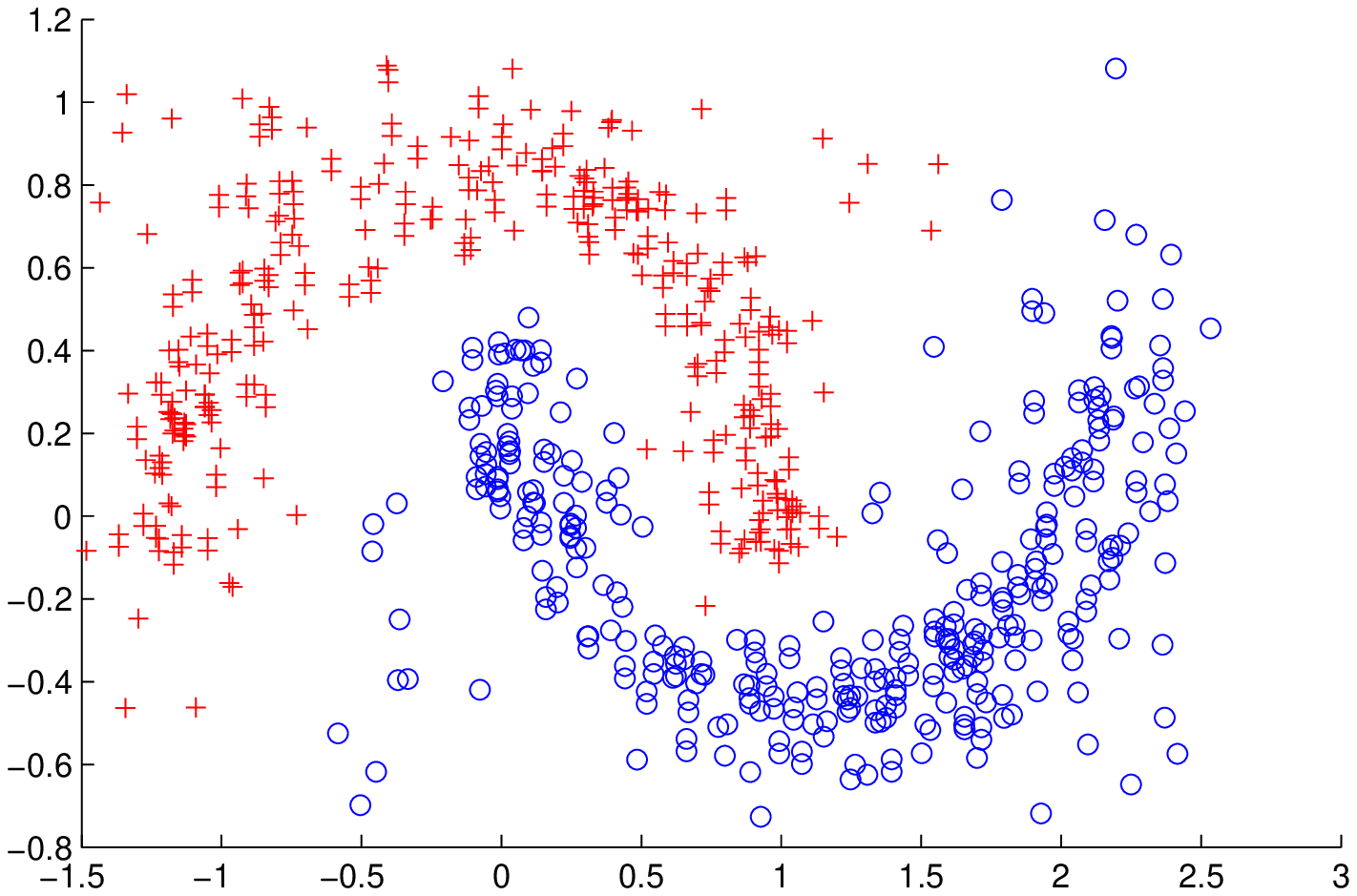}}
\caption{The partition of a noisy Double Moon sample.}\label{fig:2moon:noisy}
\end{figure*}

\subsection{Constraints from Partial Labeling: Clustering the UCI Benchmarks}
\label{sec:exp:uci}

Next we evaluate the performance of our algorithm by clustering the UCI benchmark datasets (\citet{Asuncion2007}) with constraints derived from ground truth labeling.

We chose six different datasets with class label information, namely Hepatitis, Iris, Wine, Glass, Ionosphere and Breast Cancer Wisconsin (Diagnostic). We performed 2-way clustering simply by partitioning the optimal cluster indicator according to the sign: positive entries to one cluster and negative the other. We removed the \textsc{setosa} class from the Iris data set, which is the class that is known to be well-separately from the other two. For the same reason we removed Class 3 from the Wine data set, which is well-separated from the other two. We also removed data instances with missing values. The statistics of the data sets after preprocessing are listed in Table~\ref{table:data}.

\begin{table}[t]
\centering
\caption{The UCI benchmarks}\label{table:data}
\begin{tabular}{|l|c|c|}
  \hline
  Identifier & \#Instances & \#Attributes\\
  \hline
  Hepatitis & 80 & 19\\
  Iris & 100 & 4\\
  Wine & 130 & 13\\
  Glass & 214 & 9\\
  Ionosphere & 351 & 34\\
  WDBC & 569 & 30\\
  \hline
\end{tabular}
\end{table}

For each data set, we computed the affinity matrix using the RBF kernel. To generate constraints, we randomly selected pairs of nodes that the unconstrained spectral clustering wrongly partitioned, and fill in the correct relation in $Q$ according to ground truth labels. The quality of the clustering results was measured by adjusted Rand index. Since the constraints are guaranteed to be correct, we set the threshold $\beta$ such that there will be only one feasible eigenvector, i.e. the one that best conforms to the constraint matrix $Q$.

In addition to comparing our algorithm (CSP) to unconstrained spectral clustering, we implemented two state-of-the-art techniques:
\begin{itemize}
\item Spectral Learning (SL) (\citet{Kamvar2003}) modifies the affinity matrix of the original graph directly: $A_{ij}$ is set to 1 if there is a ML between node $i$ and $j$, 0 for CL.
\item Semi-Supervised Kernel $K$-means (SSKK) (\citet{DBLP:conf/icml/KulisBDM05}) adds penalties to the affinity matrix based on the given constraints, and then performs kernel $K$-means on the new kernel to find the partition.
\end{itemize}
We also tried the algorithm proposed by \citet{Yu2001,Yu2004}, which encodes partial grouping information as a projection matrix, the subspace trick proposed by \citet{Bie2004}, and the affinity propagation algorithm proposed by \citet{Lu2008}. However, we were not able to use these algorithms to achieve better results than SL and SSKK, hence their results are not reported. \citet{Xu2005} proposed a variation of SL, where the constraints are encoded in the same way, but instead of the normalized graph Laplacian, they suggested to use the random walk matrix. We performed their algorithm on the UCI datasets, which produced largely identical results to that of SL.

We report the adjusted Rand index against the number of constraints used (ranging from $50$ to $500$) so that we can see how the quality of clustering varies when more constraints are provided. At each stop, we randomly generated 100 sets of constraints. The mean, maximum and minimum ARI of the 100 random trials are reported in Fig.~\ref{fig:uci_ari}. We also report the ratio of the constraints that were satisfied by the constrained partition in Fig.~\ref{fig:uci_sat}. The observations are:
\begin{itemize}
\item Across all six datasets, our algorithm is able to effectively utilize the constraints and improve over unconstrained spectral clustering (Baseline). On the one hand, our algorithm can quickly improve the results with a small amount of constraints. On the other hand, as more constraints are provided, the performance of our algorithm consistently increases and converges to the ground truth partition (Fig.~\ref{fig:uci_ari}).
\item Our algorithm outperforms the competitors by a large margin in terms of ARI (Fig.~\ref{fig:uci_ari}). Since we have control over the lower-bounding threshold $\alpha$, our algorithm is able to satisfy almost all the given constraints (Fig.~\ref{fig:uci_sat}).
\item The performance of our algorithm has significantly smaller variance over different random constraint sets than its competitors (Fig.~\ref{fig:uci_ari} and \ref{fig:uci_sat}), and the variance quickly diminishes as more constraints are provided. This suggests that our algorithm would perform more consistently in practice.
\item Our algorithm performs especially well on sparse graphs, i.e. Fig.~\ref{fig:uci_ari}(e)(f), where the competitors suffer. The reason is that when the graph is too sparse, it implies many ``free'' cuts that are equally good to unconstrained spectral clustering. Even after introducing a small number of constraints, the modified graph remains too sparse for SL and SSKK, which are unable to identify the ground truth partition. In contrast, these free cuts are not equivalent when judged by the constraint matrix $Q$ of our algorithm, which can easily identify the one cut that minimizes $\mathbf{v}^T\bar{Q}\mathbf{v}$. As a result, our algorithm outperforms SL and SSKK significantly in this scenario.
\end{itemize}

\begin{figure*}
\centering
\subfigure[Hepatitis]{\includegraphics*[width=0.45\linewidth]{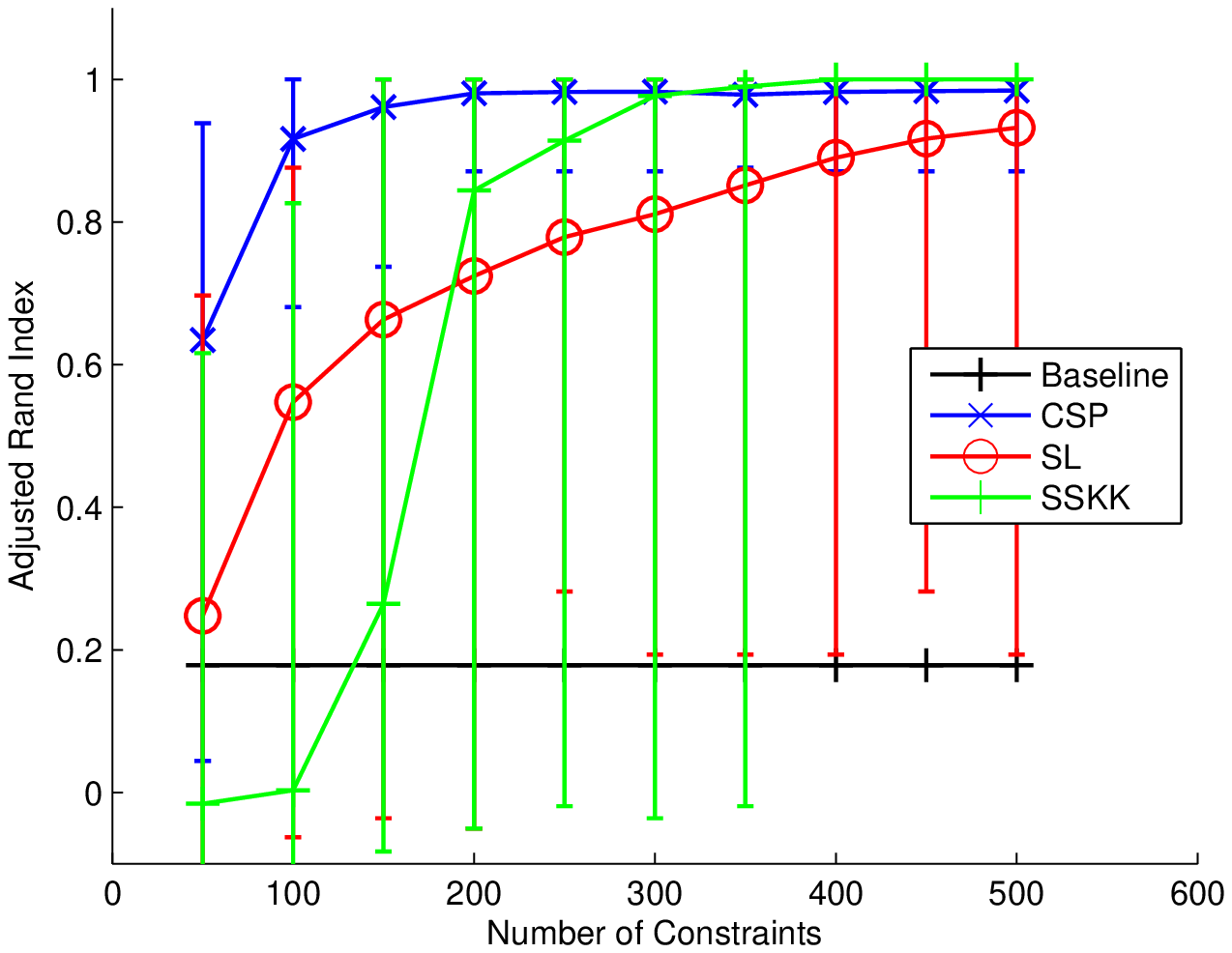}}
\subfigure[Iris]{\includegraphics*[width=0.45\linewidth]{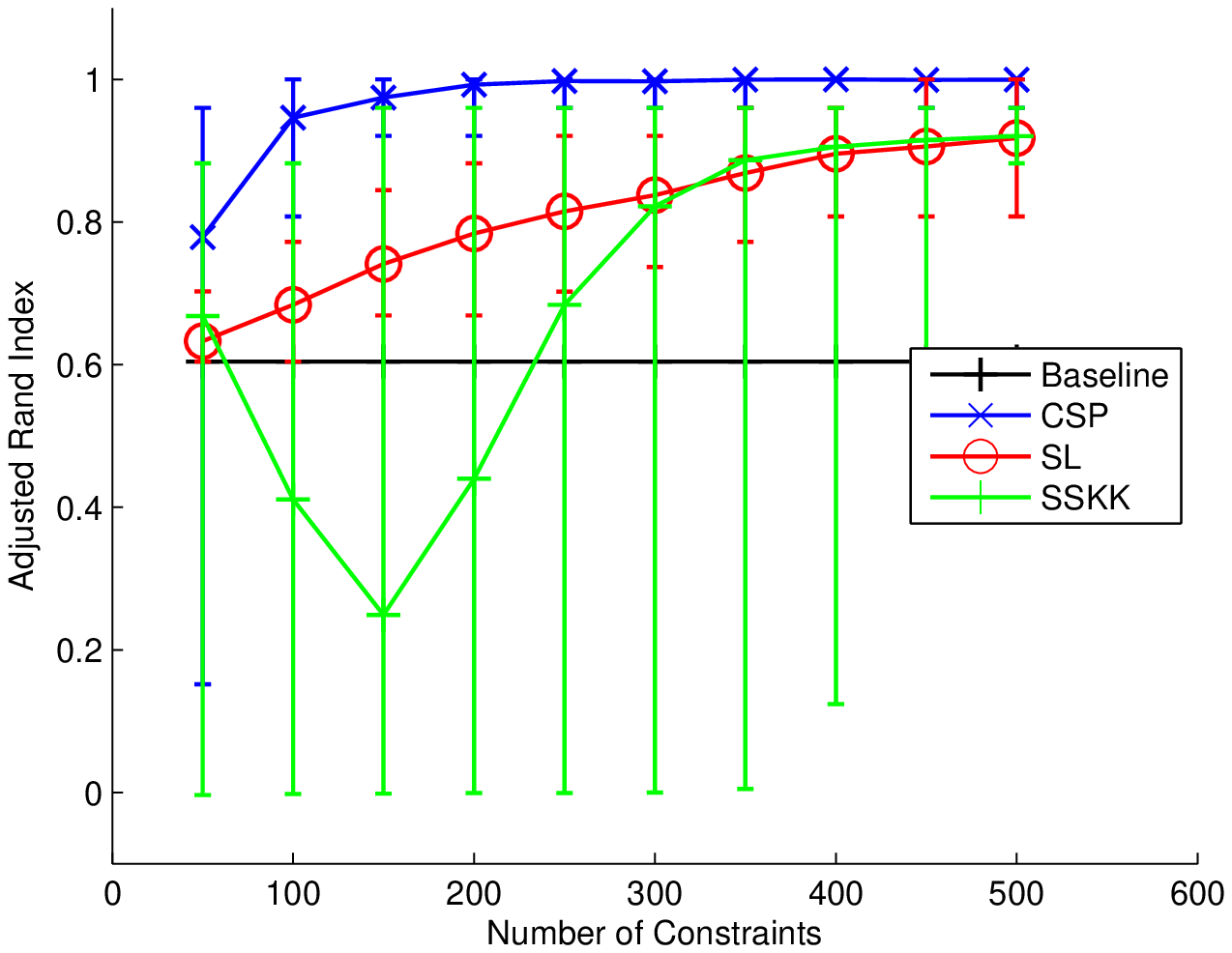}}\\
\subfigure[Wine]{\includegraphics*[width=0.45\linewidth]{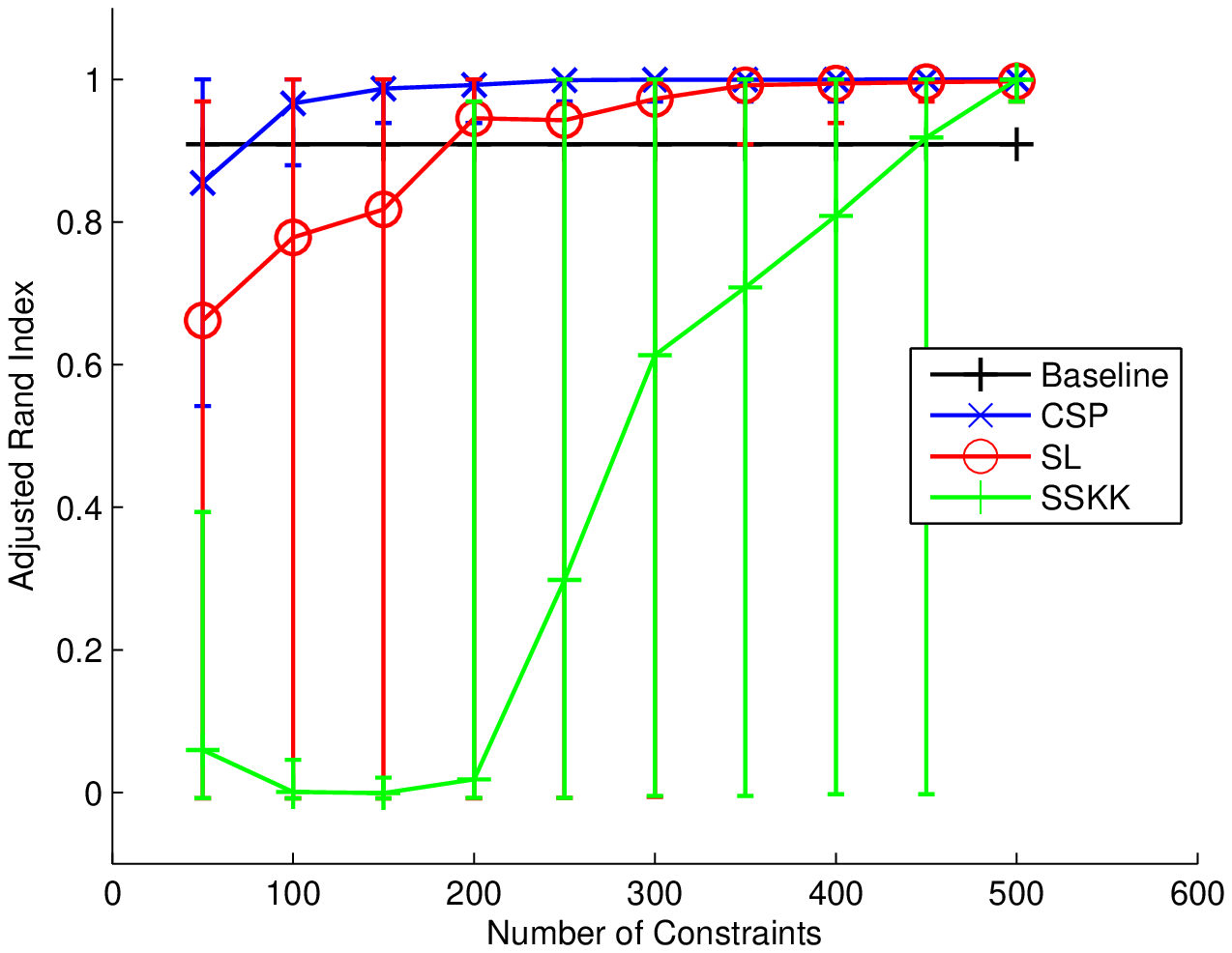}}
\subfigure[Glass]{\includegraphics*[width=0.45\linewidth]{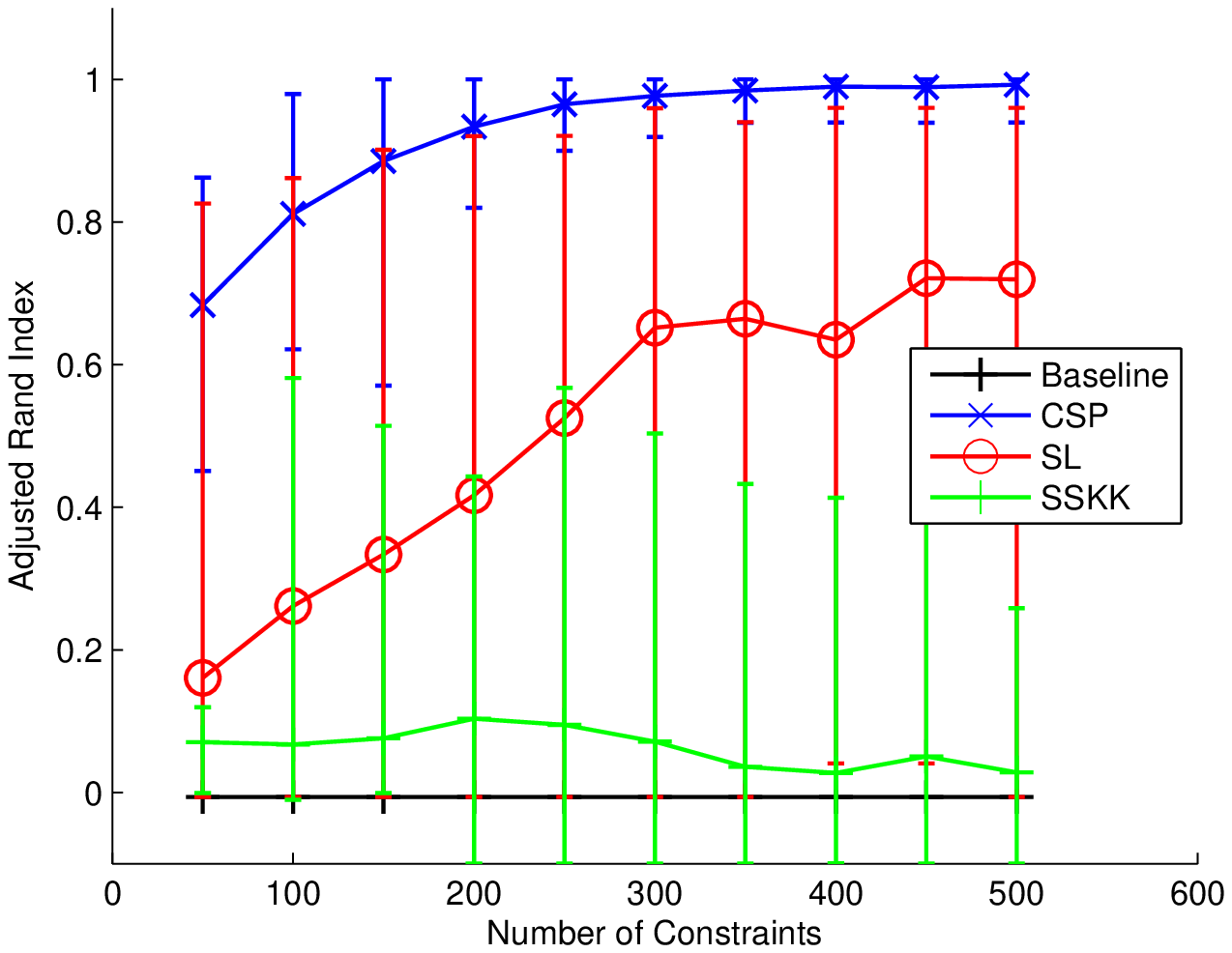}}\\
\subfigure[Ionosphere]{\includegraphics*[width=0.45\linewidth]{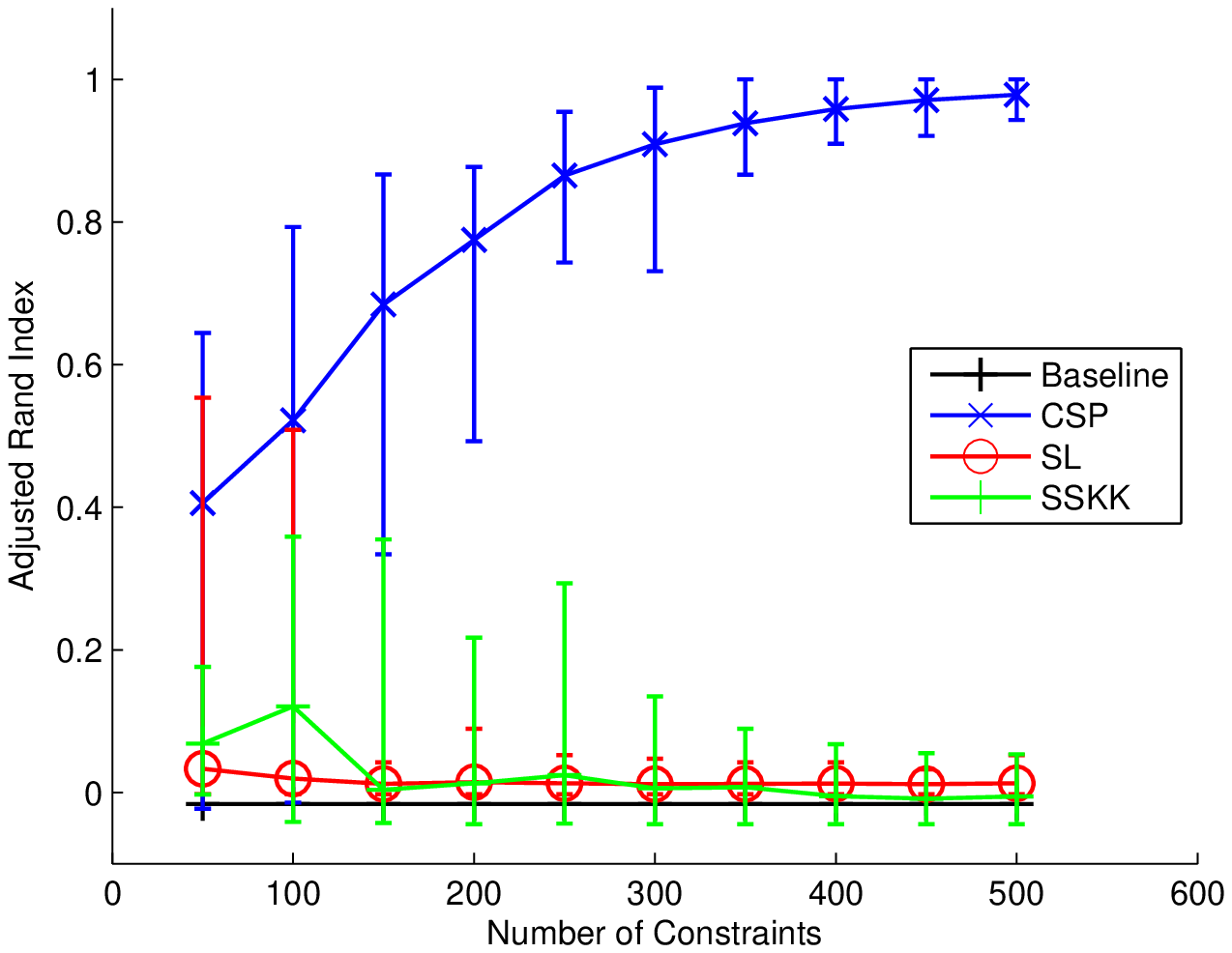}}
\subfigure[Breast Cancer]{\includegraphics*[width=0.45\linewidth]{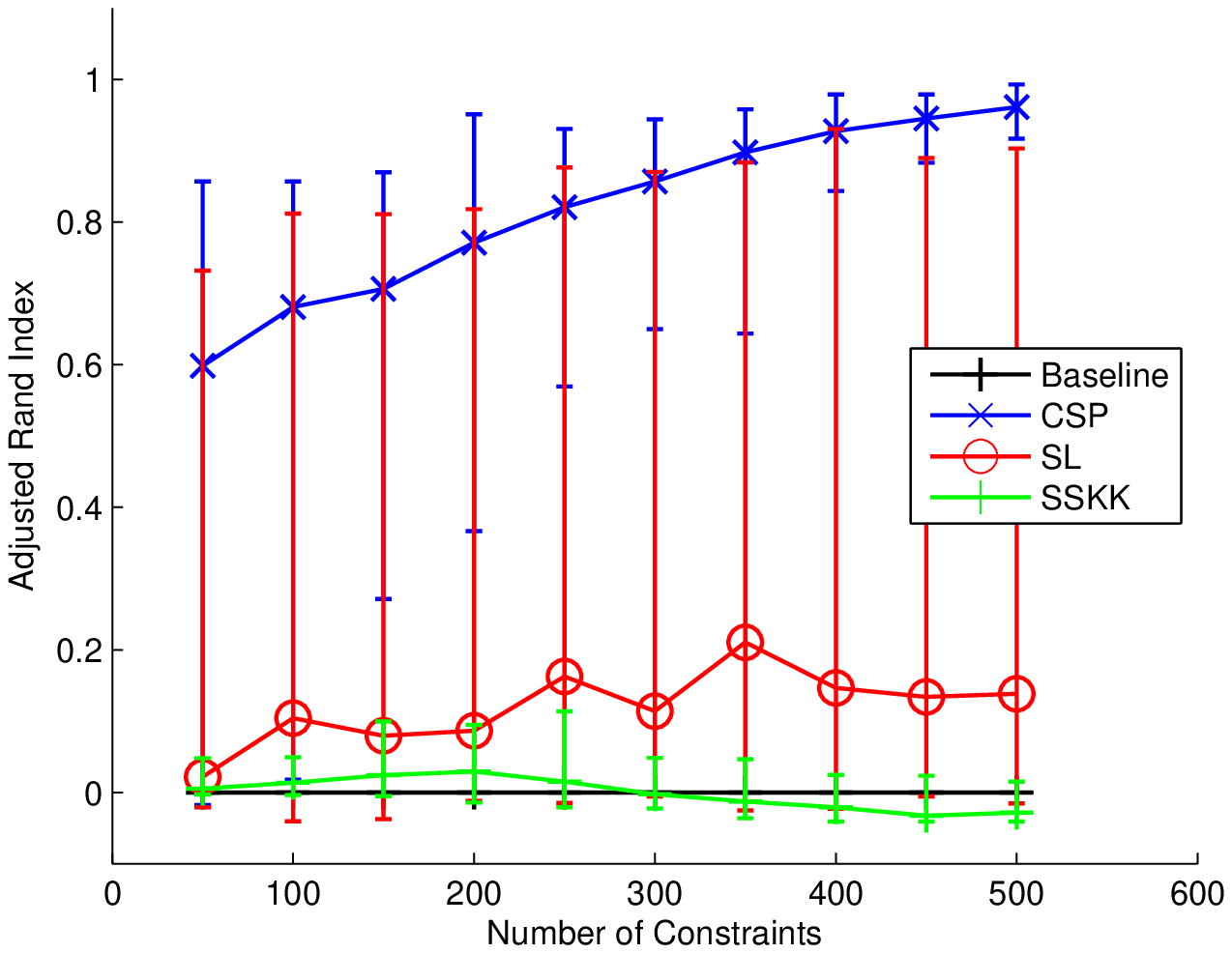}}
\caption{The performance of our algorithm (CSP) on six UCI datasets, with comparison to unconstrained spectral clustering (Baseline) and the Spectral Learning algorithm (SL). Adjusted Rand index over 100 random trials is reported (mean, min, and max).}\label{fig:uci_ari}
\end{figure*}

\begin{figure*}
\centering
\subfigure[Hepatitis]{\includegraphics*[width=0.45\linewidth]{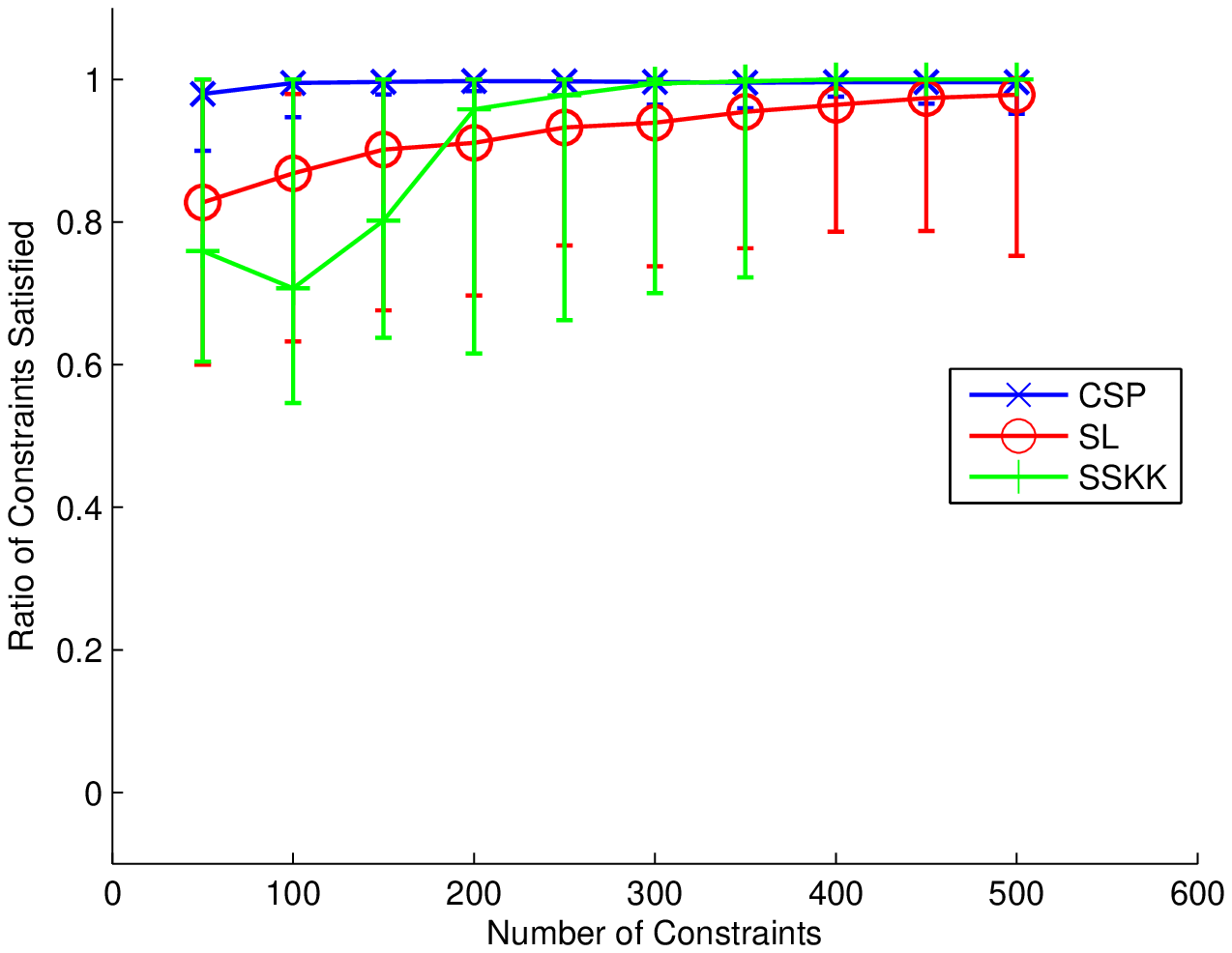}}
\subfigure[Iris]{\includegraphics*[width=0.45\linewidth]{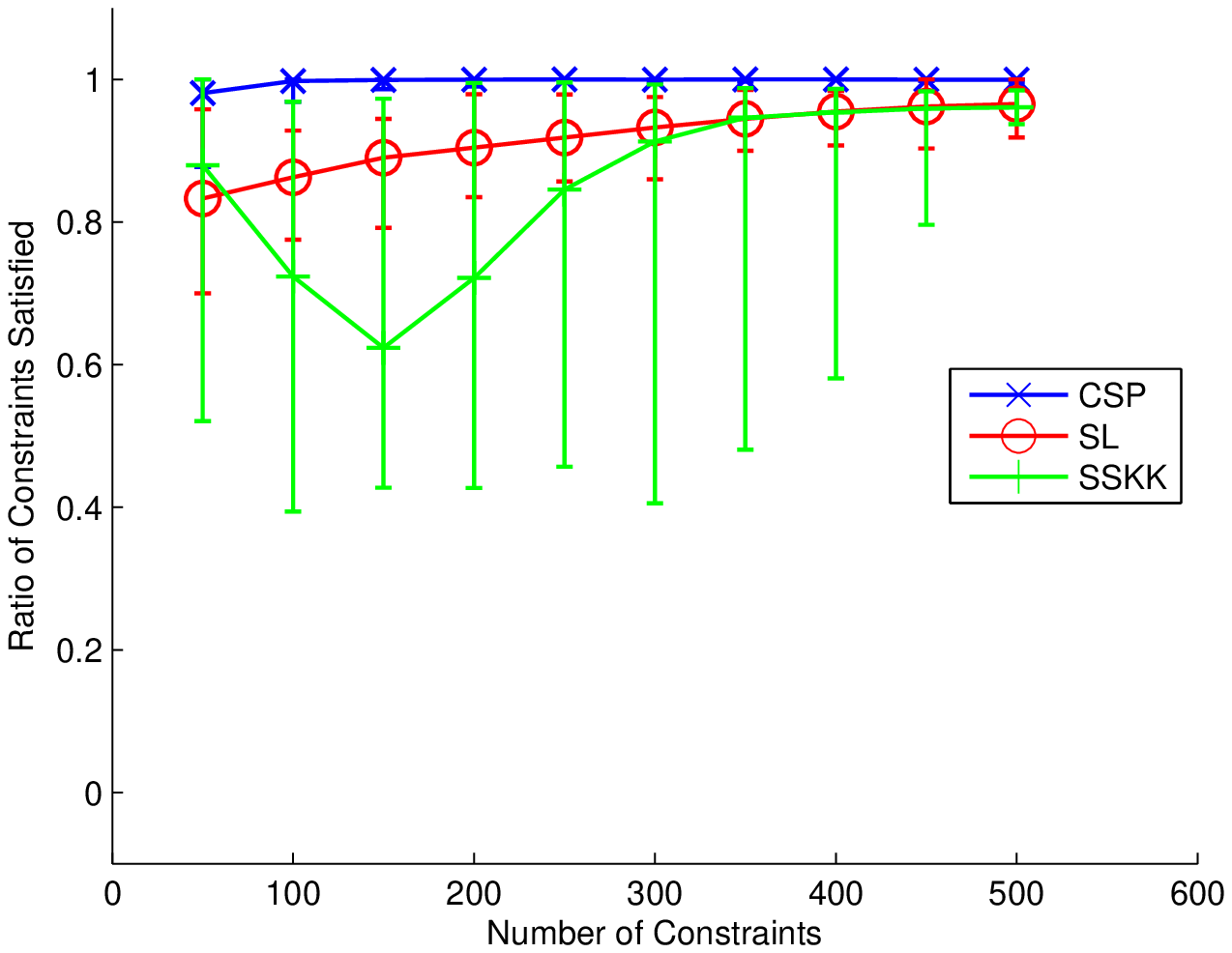}}\\
\subfigure[Wine]{\includegraphics*[width=0.45\linewidth]{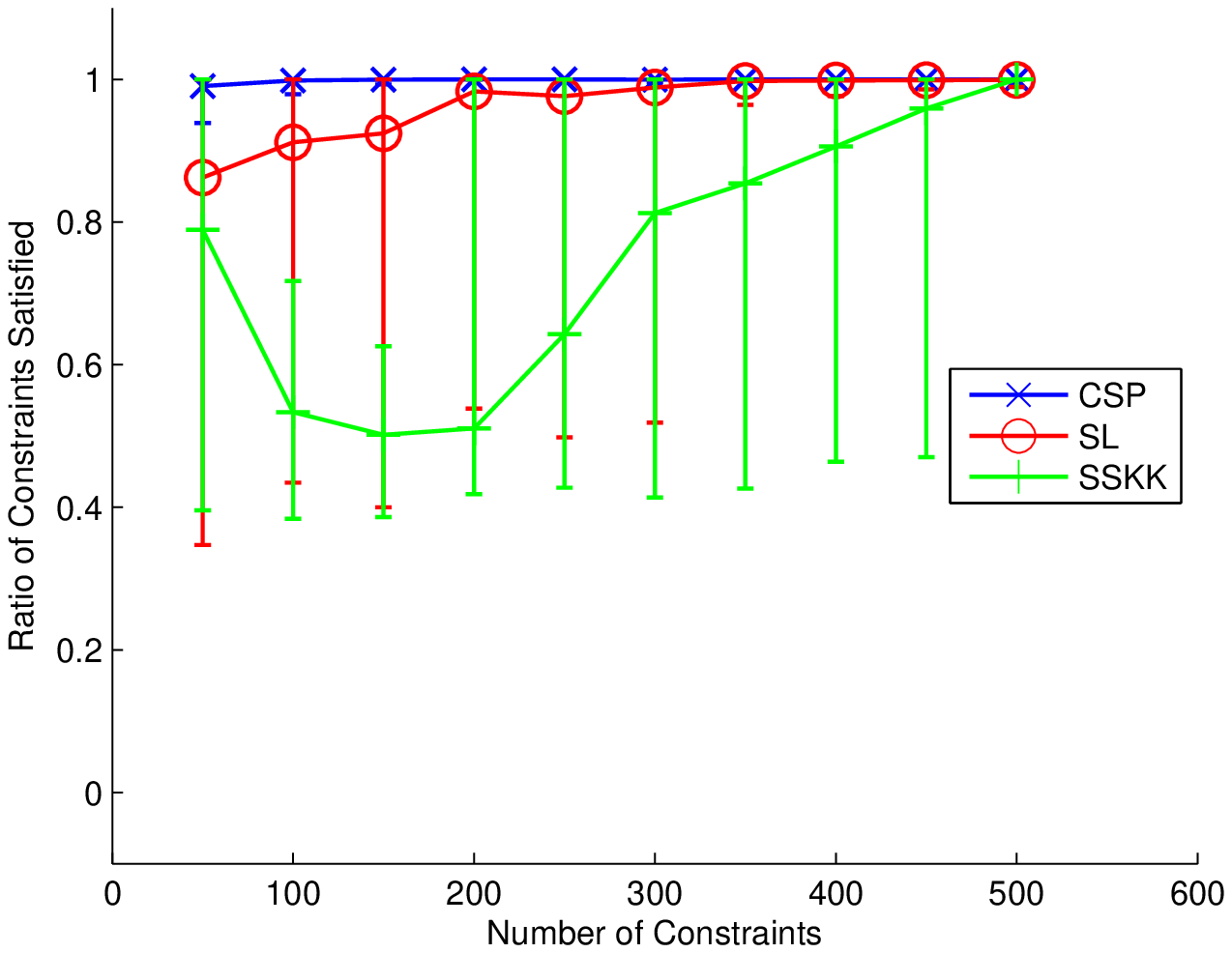}}
\subfigure[Glass]{\includegraphics*[width=0.45\linewidth]{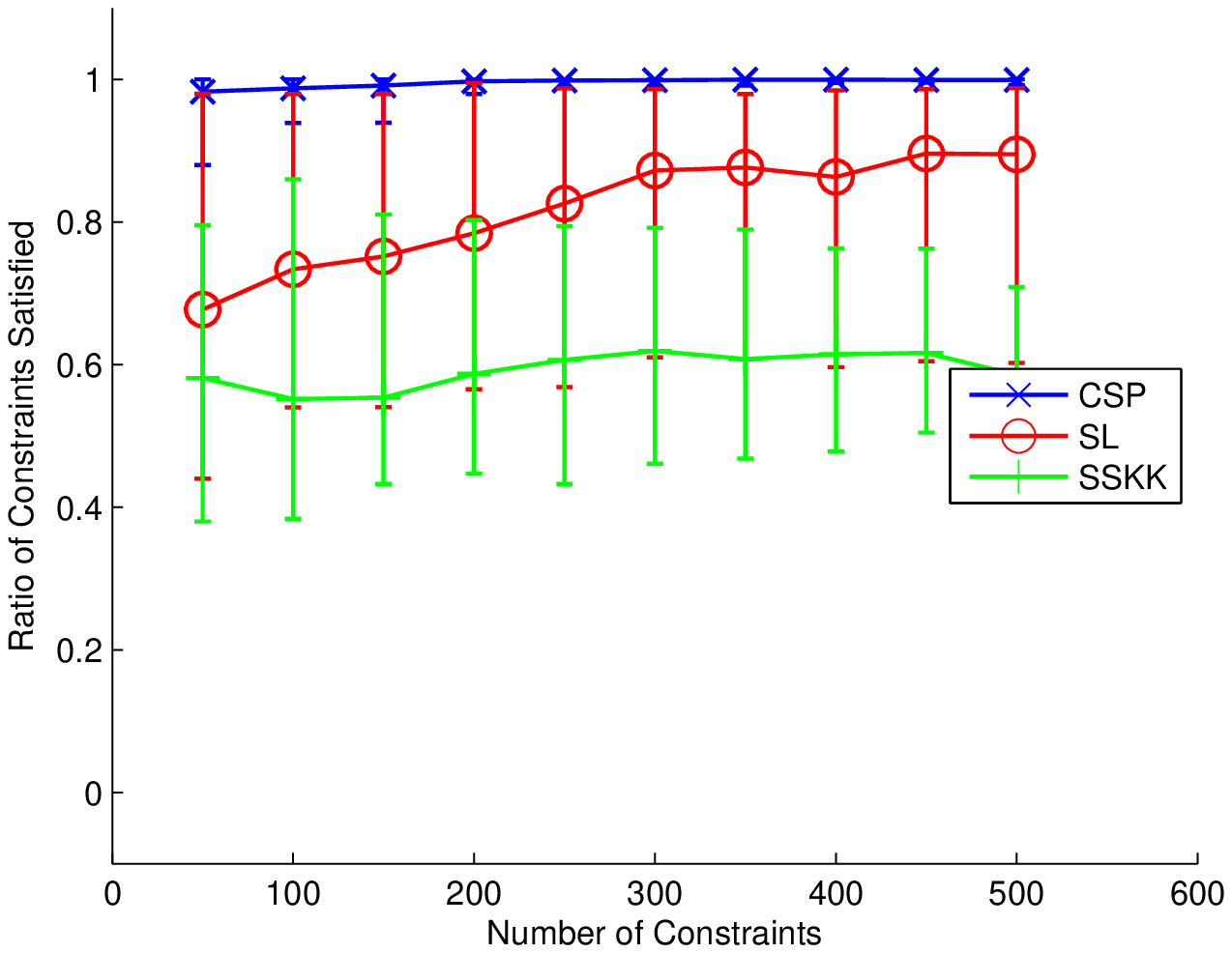}}\\
\subfigure[Ionosphere]{\includegraphics*[width=0.45\linewidth]{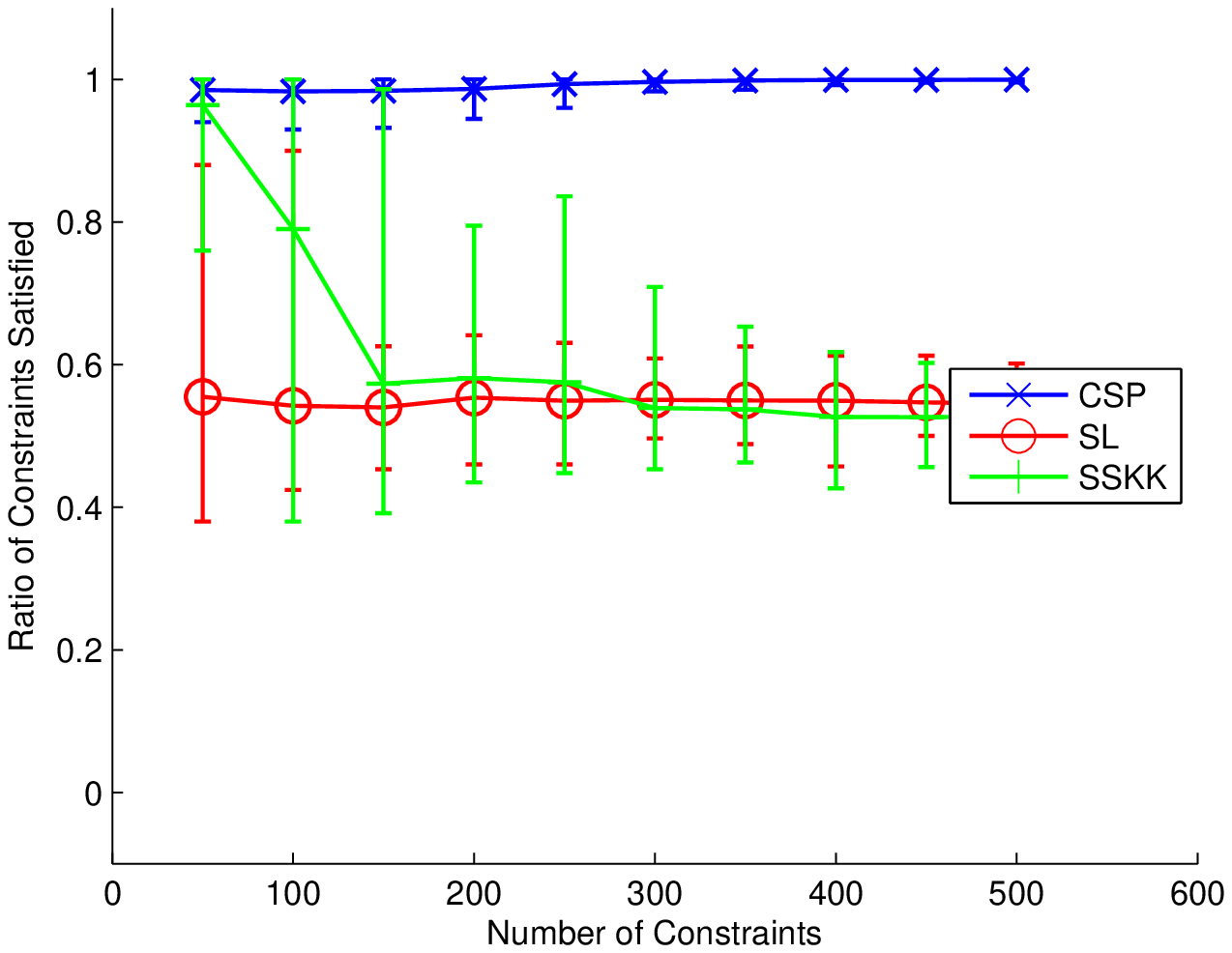}}
\subfigure[Breast Cancer]{\includegraphics*[width=0.45\linewidth]{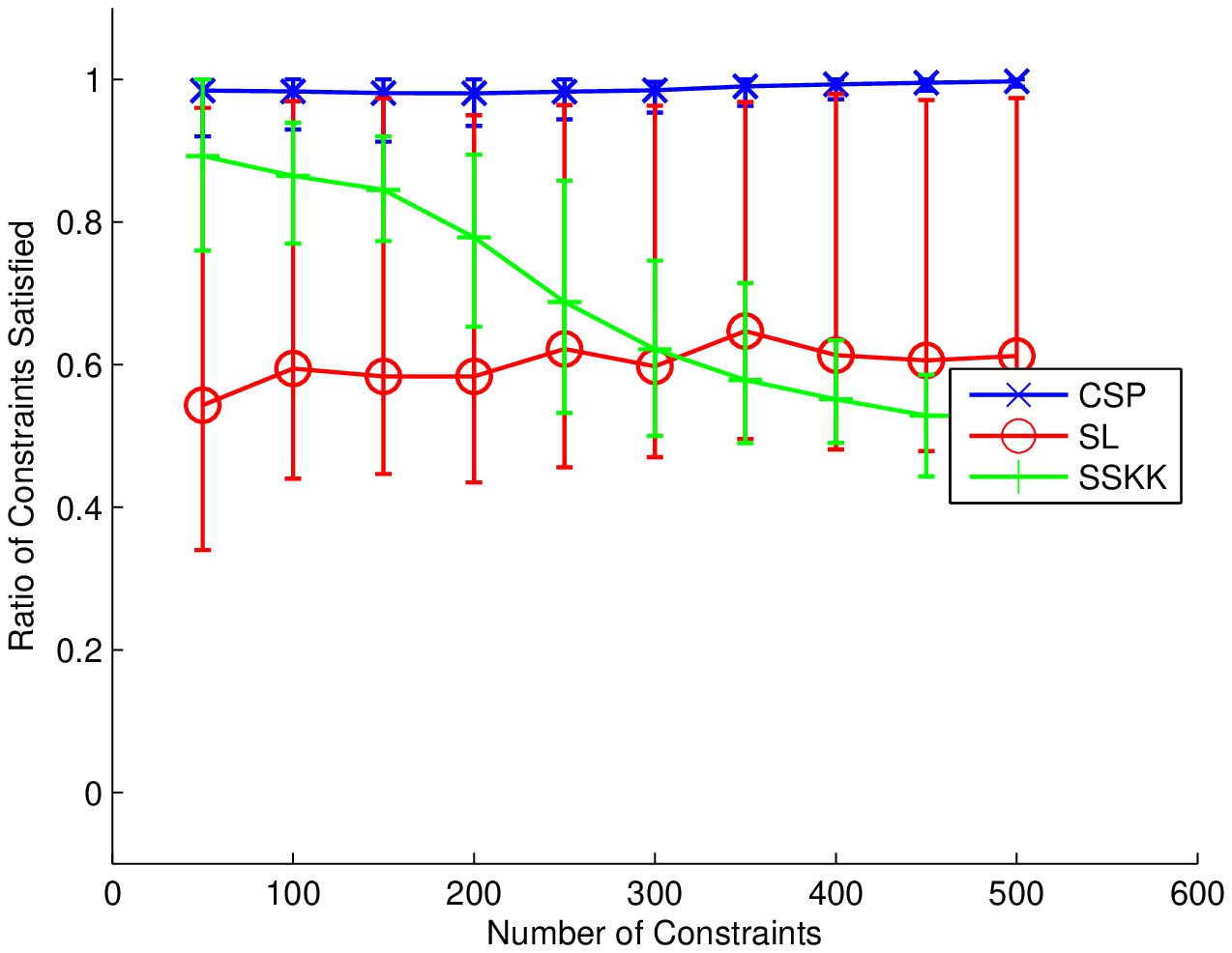}}
\caption{The ratio of constraints that are actually satisfied.}\label{fig:uci_sat}
\end{figure*}

\subsection{Constraints from Alternative Metrics: The Reuters Multilingual Dataset}
\label{sec:exp:reuters}

We test our algorithm using soft constraints derived from alternative metrics of the same set of data instances.

We used the Reuters Multilingual dataset, first introduced by \citet{AUG09}. Each time we randomly sampled 1000 documents which were originally written in one language and then translated into four others, respectively. The statistics of the dataset is listed in Table~\ref{table:reuters}. These documents came with ground truth labels that categorize them into six topics ($K=6$). We constructed one graph based on the original language, and another graph based on the translation. The affinity matrix was the cosine similarity between the tf-idf vectors of two documents. Then we used one of the two graphs as the constraint matrix $Q$, whose entries can then be viewed as soft ML constraints. We enforce this constraint matrix to the other graph to see if it can help improve the clustering. We did not compare our algorithm to existing techniques because they are unable to incorporate soft constraints.

As shown in Fig.~\ref{fig:reuters}, unconstrained spectral clustering performs better on the original version than the translated versions, which is not surprising. If we use the original version as the constraints and enforce that onto a translated version using our algorithm, we yield a constrained clustering that is not only better than the unconstrained clustering on the translated version, but also even better than the unconstrained clustering on the original version. This indicates that our algorithm is not merely a tradeoff between the original graph and the given constraints. Instead it is able to integrate the knowledge from the constraints into the original graph and achieve a better partition.

\begin{table}
\centering
\caption{The Reuters Multilingual dataset}\label{table:reuters}
\begin{tabular}{|l|c|c|}
  \hline
  Language & \#Documents & \#Words\\
  \hline
  English & 2000 & 21,531\\
  French & 2000 & 24,893\\
  German & 2000 & 34,279\\
  Italian & 2000 & 15,506\\
  Spanish & 2000 & 11,547\\
  \hline
\end{tabular}
\end{table}

\begin{figure*}
\centering
\subfigure[English Documents and Translations]{\includegraphics*[width=0.45\linewidth]{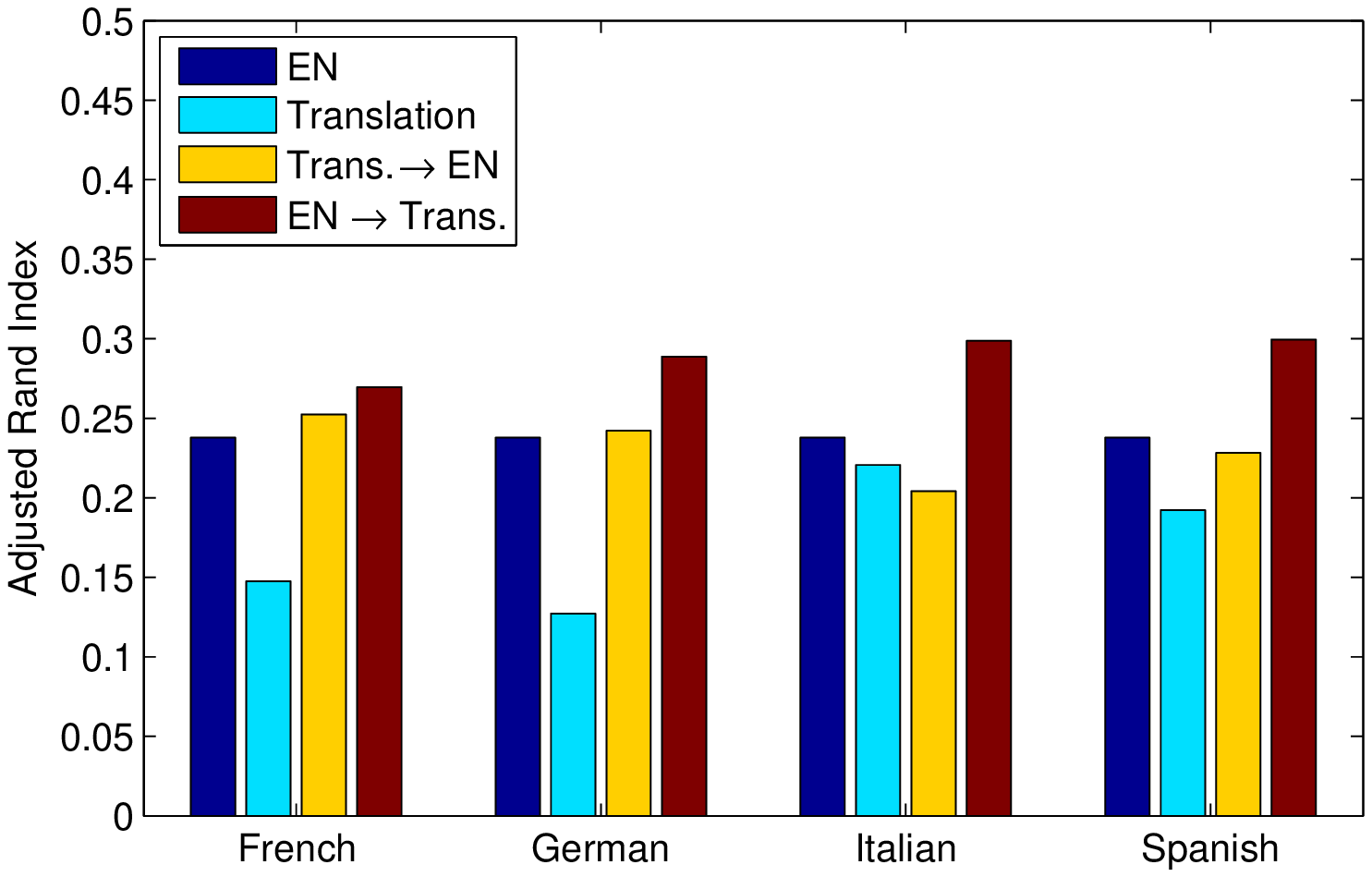}}
\subfigure[French Documents and Translations]{\includegraphics*[width=0.45\linewidth]{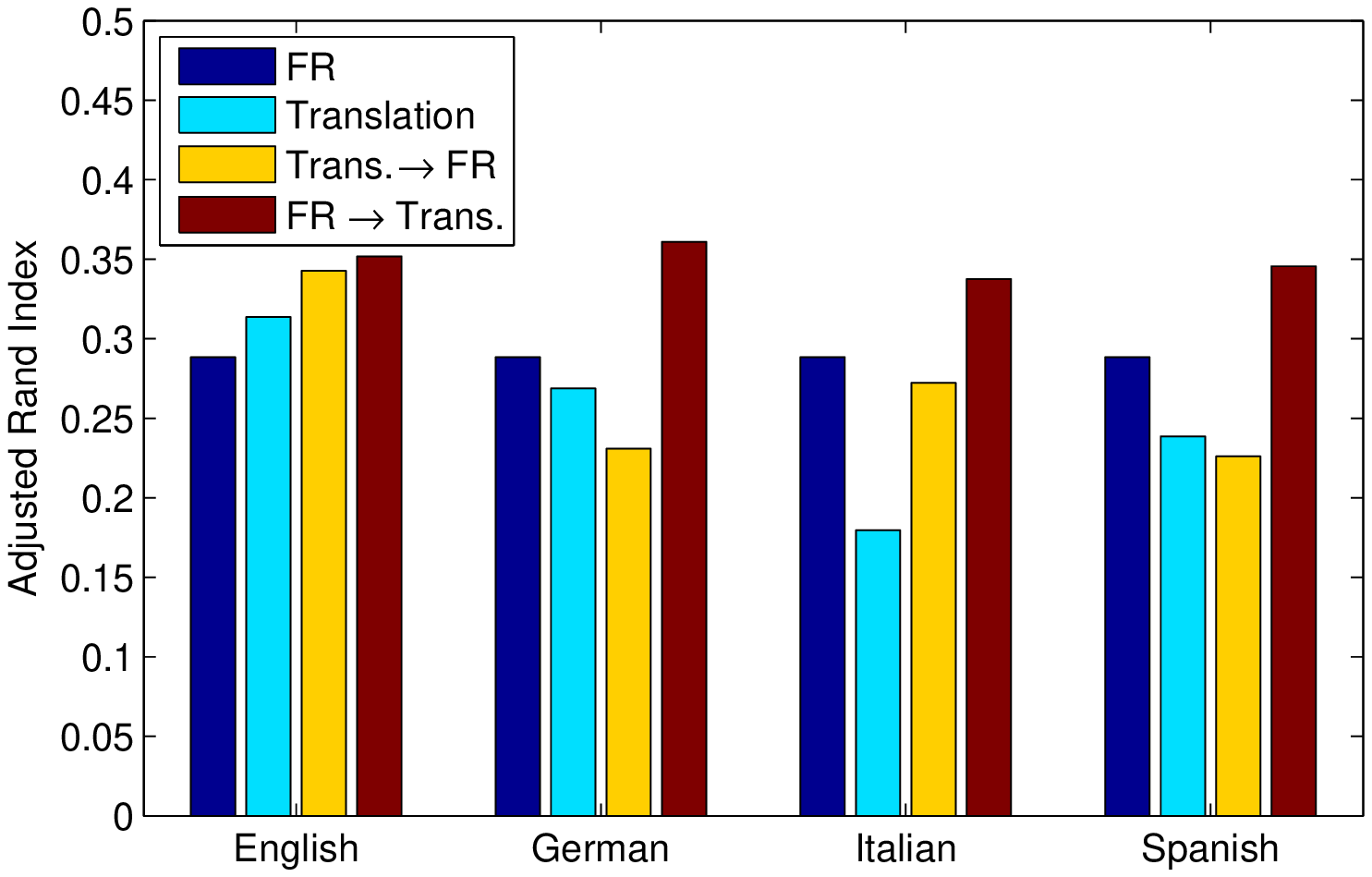}}
\caption{The performance of our algorithm on Reuters Multilingual dataset.}\label{fig:reuters}
\end{figure*}

\subsection{Transfer of Knowledge: Resting-State fMRI Analysis}
\label{sec:exp:transfer}

Finally we apply our algorithm to transfer learning on the resting-state fMRI data.

An fMRI scan of a person consists of a sequence of 3D images over time. We can construct a graph from a given scan such that a node in the graph corresponds to a voxel in the image, and the edge weight between two nodes is (the absolute value of) the correlation between the two time sequences associated with the two voxels. Previous work has shown that by applying spectral clustering to the resting-state fMRI we can find the substructures in the brain that are periodically and simultaneously activated over time in the resting state, which may indicate a network associated with certain functions (\citet{10.1371/journal.pone.0002001}).

One of the challenges of resting-state fMRI analysis is instability. Noise can be easily introduced into the scan result, e.g. the subject moved his/her head during the scan, the subject was not at resting state (actively thinking about things during the scan), etc. Consequently, the result of spectral clustering becomes instable. If we apply spectral clustering to two fMRI scans of \emph{the same} person on two different days, the normalized min-cuts on the two different scans are so different that they provide little insight into the brain activity of the subject (Fig.~\ref{fig:fmri}(a) and (b)). To overcome this problem, we use our formulation to transfer knowledge from Scan 1 to Scan 2 and get a constrained cut, as shown in Fig.~\ref{fig:fmri}(c). This cut represents what the two scans agree on. The pattern captured by Fig.~\ref{fig:fmri}(c) is actually the \emph{default mode network} (DMN), which is the network that is periodically activated at resting state (Fig.~\ref{fig:fmri}(d) shows the idealized DMN as specified by domain experts).

\begin{figure*}
\centering
\subfigure[Ncut of Scan 1]{\includegraphics*[width=0.45\linewidth]{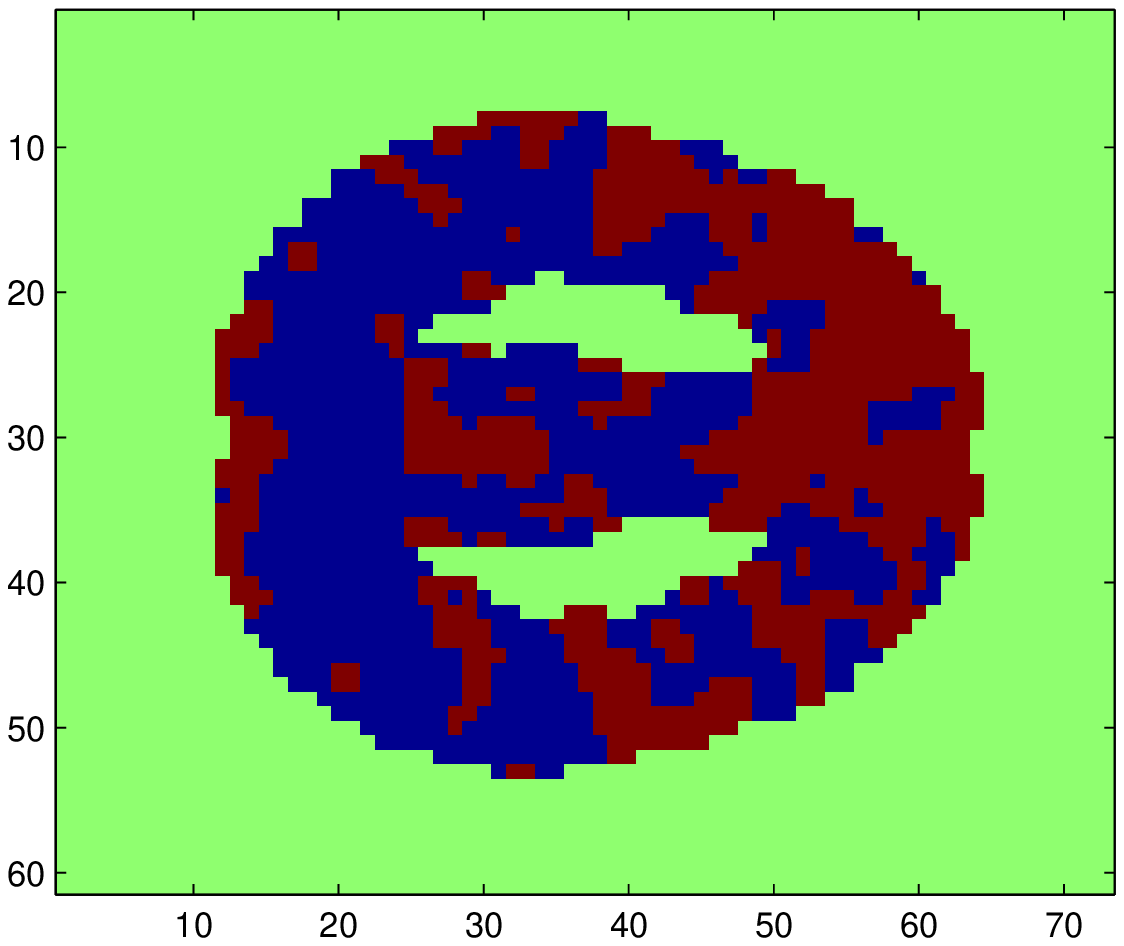}}
\subfigure[Ncut of Scan 2]{\includegraphics*[width=0.45\linewidth]{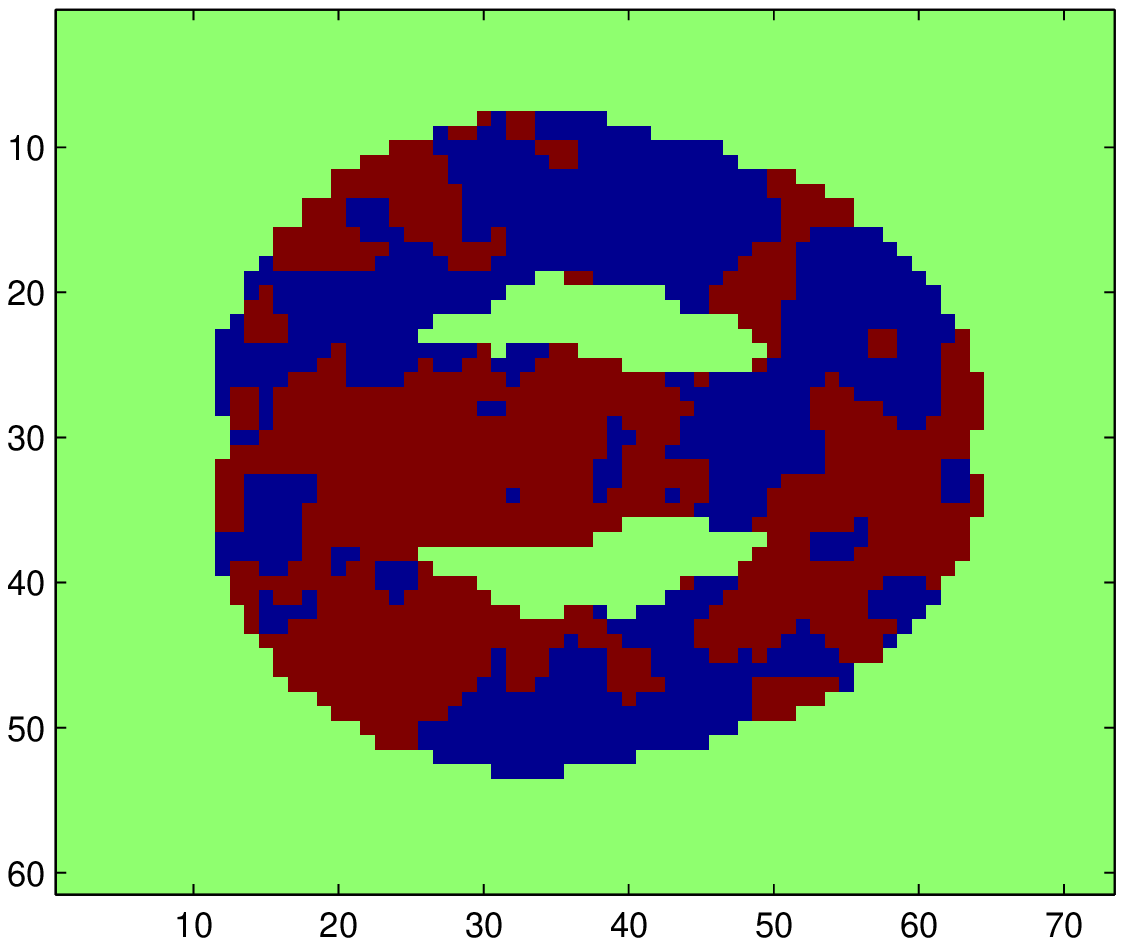}}\\
\subfigure[Constrained cut by transferring Scan 1 to 2]{\includegraphics*[width=0.45\linewidth]{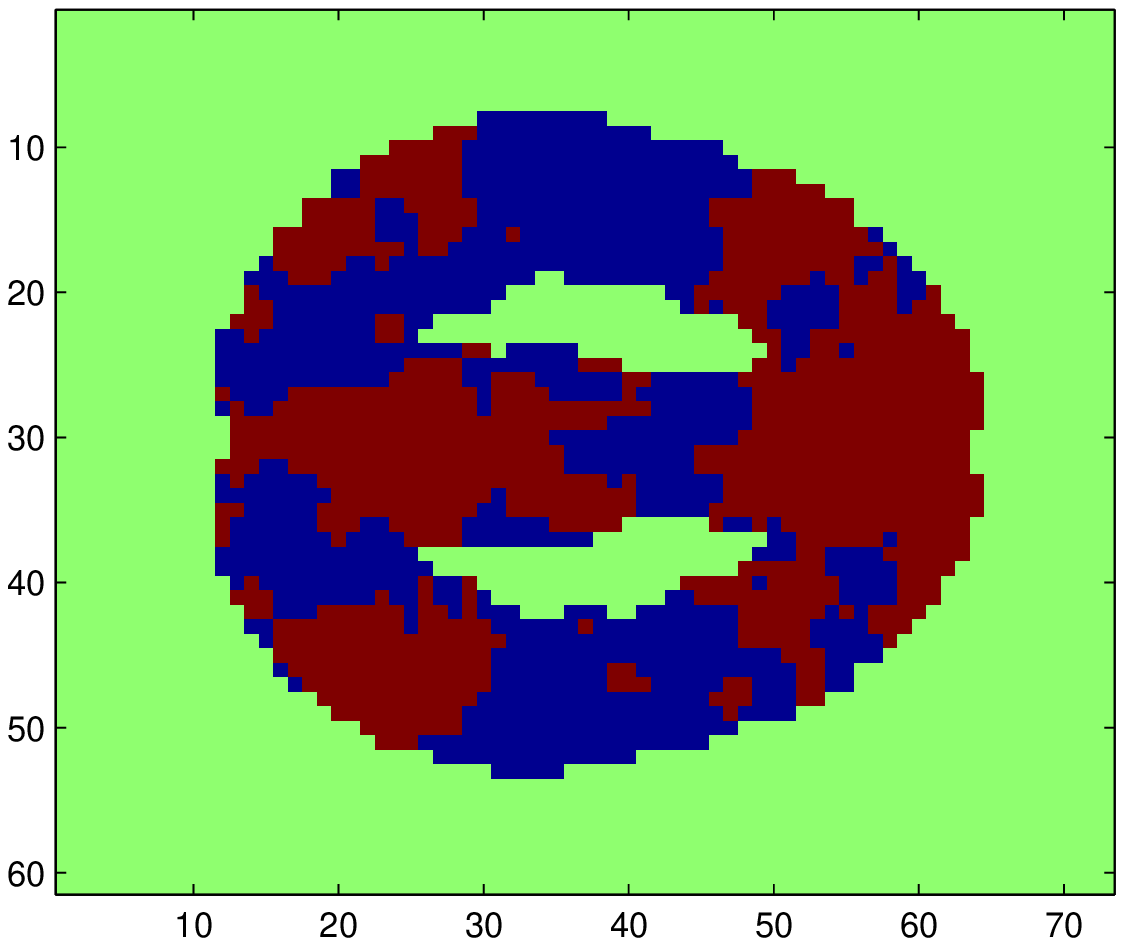}}
\subfigure[An idealized default mode network]{\includegraphics*[width=0.45\linewidth]{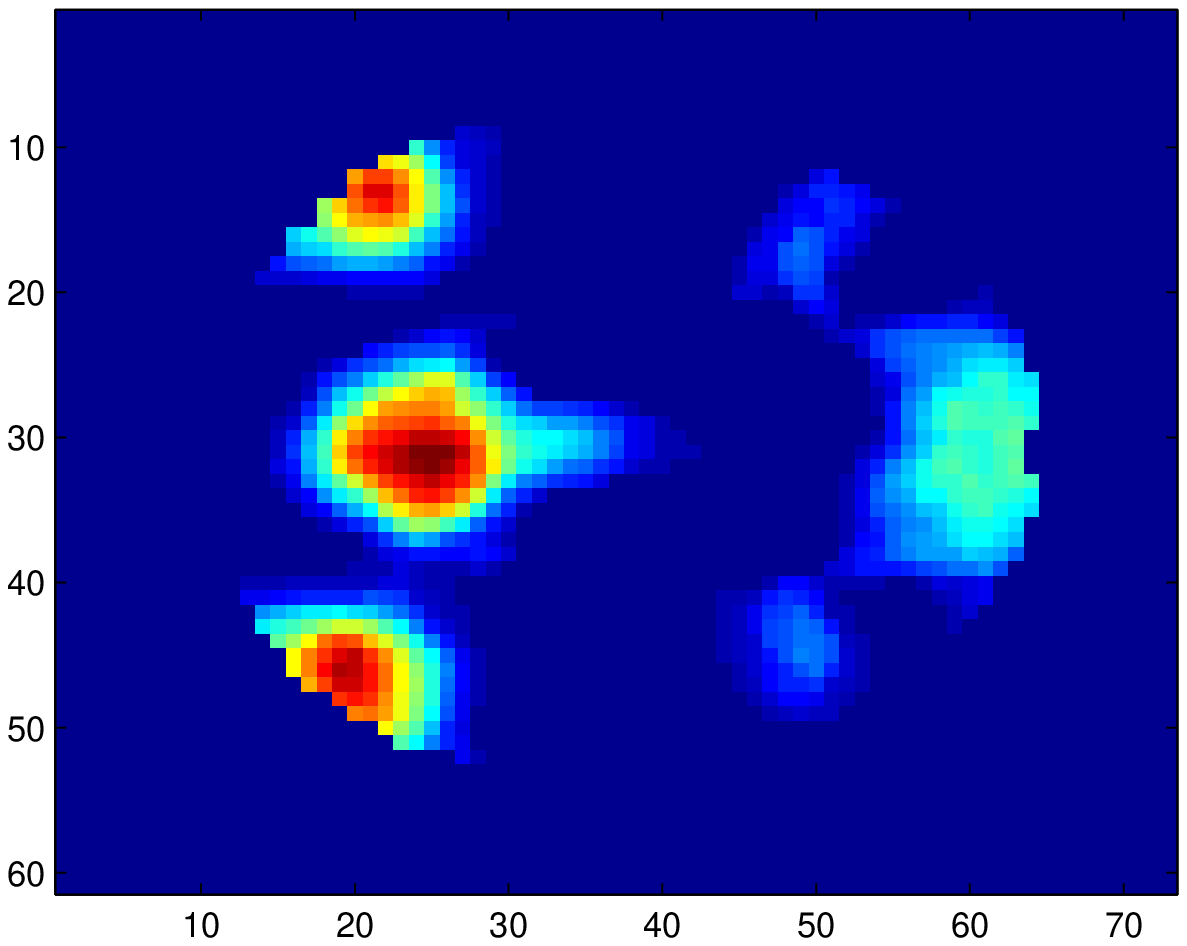}}
\caption{Transfer learning on fMRI scans.}\label{fig:fmri}
\end{figure*}

To further illustrate the practicability of our approach, we transfer the idealized DMN in Fig.~\ref{fig:fmri}(d) to a set of fMRI scans of elderly subjects. The dataset was collected and processed within the research program of the University of California at Davis Alzheimer's Disease Center (UCD ADC). The subjects were categorized into two groups: those diagnosed with cognitive syndrome (20 individuals) and those without cognitive syndrome (11 individuals). For each individual scan, we encode the idealized DMN into a constraint matrix (using the RBF kernel), and enforce the constraints onto the original fMRI scan. We then compute the cost of the constrained cut that is the most similar to the DMN. If the cost of the constrained cut is high, it means there is great disagreement between the original graph and the given constraints (the idealized DMN), and vice versa. In other words, the cost of the constrained cut can be interpreted as the cost of transferring the DMN to the particular fMRI scan.

In Fig.~\ref{fig:fmri:str}, we plot the costs of transferring the DMN to both subject groups. We can clearly see that the costs of transferring the DMN to people without cognitive syndrome tend to be lower than to people with cognitive syndrome. This conforms well to the observation made in a recent study that the DMN is often disrupted for people with the Alzheimer's disease (\citet{NYAS:NYAS1124011}).

\begin{figure}
\centering
\includegraphics*[width=0.8\linewidth]{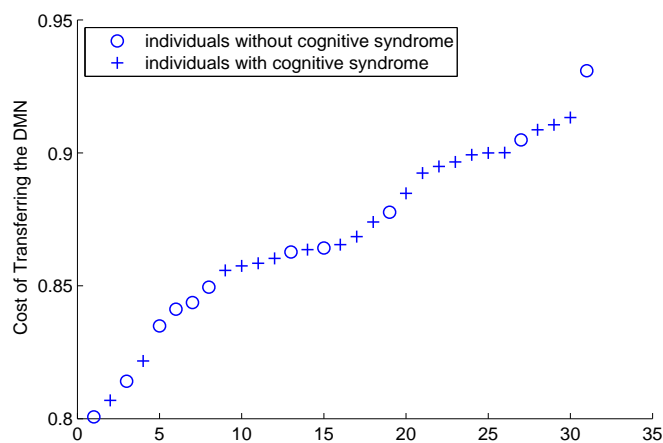}
\caption{The costs of transferring the idealized default mode network to the fMRI scans of two groups of elderly individuals.}\label{fig:fmri:str}
\end{figure}

\section{Conclusion}
\label{sec:conclusion}

In this work we proposed a principled and flexible framework for constrained spectral clustering that can incorporate large amounts of both hard and soft constraints. The flexibility of our framework lends itself to the use of all types of side information: pairwise constraints, partial labeling, alternative metrics, and transfer learning. Our formulation is a natural extension to unconstrained spectral clustering and can be solved efficiently using generalized eigendecomposition. We demonstrated the effectiveness of our approach on a variety of datasets: the synthetic Two-Moon dataset, image segmentation, the UCI benchmarks, the multilingual Reuters documents, and resting-state fMRI scans. The comparison to existing techniques validated the advantage of our approach.

\section{Acknowledgments}

We gratefully acknowledge support of this research via ONR grants N00014-09-1-0712 Automated Discovery and Explanation of Event Behavior, N00014-11-1-0108 Guided Learning in Dynamic Environments and NSF Grant NSF IIS-0801528 Knowledge Enhanced Clustering.

%\bibliographystyle{spbasic}
%\bibliography{csp}

\end{document}